\documentclass{article}

\let\originaladdcontentsline\addcontentsline

\usepackage[preprint]{colm2026_conference}
\usepackage{amsmath,amssymb,amsfonts}
\usepackage{amsthm}
\usepackage{booktabs}
\usepackage{multirow}
\usepackage{graphicx}
\usepackage{subcaption}
\usepackage{hyperref}
\usepackage{url}
\usepackage{xcolor}
\usepackage{colortbl}
\usepackage{microtype}
\usepackage{algorithm}
\usepackage{algpseudocode}
\usepackage{wrapfig}
\usepackage{enumitem}
\usepackage{titletoc}
\usepackage{lineno}

\definecolor{bestcolor}{RGB}{230,245,230}
\definecolor{ourcolor}{RGB}{240,240,255}
\definecolor{darkblue}{rgb}{0, 0, 0.5}
\hypersetup{colorlinks=true, citecolor=darkblue, linkcolor=darkblue, urlcolor=darkblue}

\theoremstyle{definition}
\newtheorem{definition}{Definition}
\theoremstyle{plain}
\newtheorem{lemma}{Lemma}
\newtheorem{proposition}{Proposition}
\title{Respecting Self-Uncertainty in On-Policy Self-Distillation for Efficient LLM Reasoning}

\author{%$\footnotemark[1] \quad
  Junlong Ke$^{2}$\footnotemark[1] \quad
  Zichen Wen$^{1}$\footnotemark[1] \quad
  Weijia Li$^{2, 3}$ \quad
  Conghui He$^{3}$\footnotemark[2] \quad
  Linfeng Zhang$^{1}$\footnotemark[2]  \\
  $^{1}$Shanghai Jiao Tong University \quad
  $^{2}$Tsinghua University \quad
  $^{3}$Shanghai AI Laboratory \quad
}

\begin{document}

\ifcolmsubmission
\linenumbers
\fi

\maketitle
\lhead{Preprint.}

{
\renewcommand{\thefootnote}{\fnsymbol{footnote}}
\footnotetext[1]{Equal contribution.}
\footnotetext[2]{Corresponding authors.}
}
\begin{abstract}
On-policy self-distillation trains a reasoning model on its own rollouts while a teacher, often the same model conditioned on privileged context, provides dense token-level supervision. Existing objectives typically weight the teacher's token-level signal uniformly across a chain-of-thought sequence, despite substantial variation in the entropy of the teacher's predictive distribution. We propose \textbf{EGRSD} (Entropy-Guided Reinforced Self-Distillation), which unifies token-level updates through three signals: a reward-grounded \emph{direction}, a teacher--student likelihood-ratio \emph{magnitude}, and the proposed teacher-entropy \emph{confidence} gate that down-weights high-entropy token positions while maintaining a nonzero lower bound on every token weight. We further introduce \textbf{CL-EGRSD}, a causal-lookahead variant that distinguishes sustained high-entropy spans from transient high-entropy positions whose following context rapidly becomes low entropy. Experiments with Qwen3-4B and Qwen3-8B in thinking mode show that EGRSD and CL-EGRSD advance the accuracy--length frontier among the compared trainable methods. 
\end{abstract}
% On-policy self-distillation trains a reasoning model on its own rollouts while a teacher, often the same model conditioned on privileged context, provides dense token-level supervision. Existing objectives typically weight the teacher's token-level signal uniformly across a chain-of-thought sequence, despite substantial variation in the entropy of the teacher's predictive distribution. We propose **EGRSD** (Entropy-Guided Reinforced Self-Distillation), which unifies token-level updates through three signals: a reward-grounded *direction*, a teacher–student likelihood-ratio *magnitude*, and the proposed teacher-entropy *confidence* gate that down-weights high-entropy token positions while maintaining a nonzero lower bound on every token weight. We further introduce **CL-EGRSD**, a causal-lookahead variant that distinguishes sustained high-entropy spans from transient high-entropy positions whose following context rapidly becomes low entropy. Experiments with Qwen3-4B and Qwen3-8B in thinking mode show that EGRSD and CL-EGRSD advance the accuracy–length frontier among the compared trainable methods.

% \clearpage

\section{Introduction}
\label{sec:intro}

Recent large language models exhibit strong multi-step inference capabilities, but chain-of-thought (CoT) reasoning~\citep{wei2022chain,openai2024o1} often entails the excessive generation of intermediate reasoning tokens. Reasoning-optimized checkpoints frequently produce redundant verification loops, simulated self-correction markers, and repeated intermediate derivations. This verbosity increases inference latency and cost, motivating methods that preserve reasoning accuracy while reducing unnecessary computation.

On-policy self-distillation is a natural approach to this problem. Instead of imitating fixed offline demonstrations, the student samples its own reasoning trajectories and receives dense token-level supervision from a privileged teacher~\citep{opsd,opsdc,yang2026rlsd}. The teacher may be the same base model conditioned on a reference solution, so the student learns from feedback on trajectories it actually visits. This avoids the train--test mismatch of offline distillation and gives a supervision signal at every completion token rather than only at the final answer.

Dense feedback, however, is not the same as reliable feedback. A reasoning completion contains heterogeneous token positions: some are low-entropy deterministic computation (arithmetic continuation, expression simplification, equation closure), while others are high-entropy branching points where the distribution assigns probability mass to multiple valid continuations (induction versus enumeration, revising a previous derivation, or discourse-level transitions). Prior work has observed analogous heterogeneity in code self-distillation~\citep{zhang2026ssd}. A privileged teacher can be sharply peaked at the former positions and diffuse at the latter, and uniformly weighting these signals risks over-emphasizing high-variance supervision. In long mathematical reasoning, we additionally identify a third regime: \emph{transient} high-entropy positions whose following context rapidly becomes low-entropy. 
\begin{wrapfigure}{r}{0.4\linewidth}
\centering
% \vspace{-9pt}
\includegraphics[width=\linewidth]{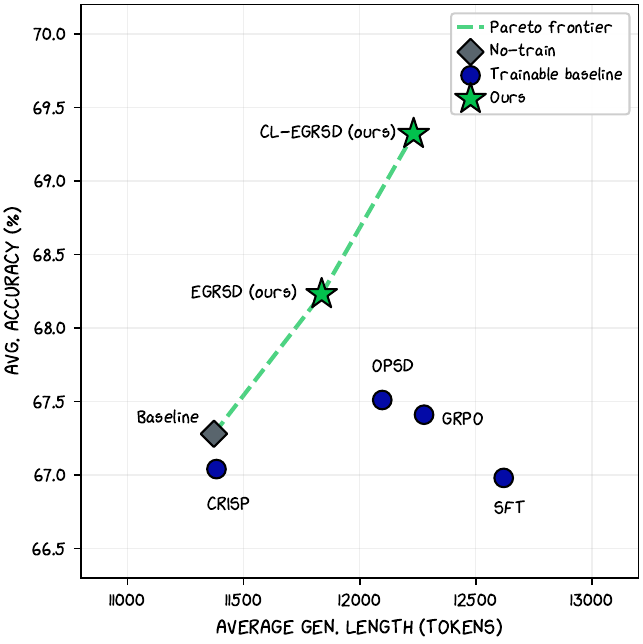}
\vspace{-21pt}
\caption{Accuracy--length trade-off on Qwen3-8B: EGRSD and CL-EGRSD (ours) extend the Pareto frontier. All trainable baselines are dominated.}
\label{fig:teaser}
\vspace{-20pt}
\end{wrapfigure}
These are strategy-shift \emph{pivots}, not sustained branching \emph{forks}, and blindly suppressing all high-entropy tokens would destroy the transition signal that the pivots carry.

This issue is especially important in self-distillation. Unlike offline distillation with a superior external model, the privileged teacher here is the base policy conditioned on augmented context. Its entropy therefore measures the concentration of the teacher's next-token predictive distribution under the privileged view. 
We refer to this distributional concentration as \emph{teacher confidence}. Because the teacher is conditioned on privileged information, this confidence acts as a practical proxy for the reliability of token-level supervision. Prior analysis further suggests that \emph{suppressing} high-entropy teacher tokens can degrade reasoning~\citep{kim2026epistemic}, motivating a non-zero lower bound in any confidence gate we introduce. Together with the outcome-reward sign and the teacher--student log-likelihood ratio used by recent direction-aware objectives, teacher predictive entropy provides a third, previously unused signal for token-level self-distillation. Concurrent work (RLSD~\citep{yang2026rlsd}) demonstrates the value of decoupling direction and magnitude. EGRSD adds the entropy-based confidence component on top of this decomposition.

We propose \textbf{Entropy-Guided Reinforced Self-Distillation} (EGRSD). For each token in a rollout, EGRSD computes the privileged-teacher entropy, normalizes it within the minibatch, and down-weights the direction-aware token update by a multiplicative confidence gate $\omega_{i,t}\in[0.1,1]$ (formalized in \S\ref{sec:egrsd}). Low-entropy computation tokens pass through with nearly full weight, while high-entropy positions are attenuated but retain a non-zero floor so branching positions with valid continuation diversity are not discarded.

A second variant, \textbf{CL-EGRSD}, addresses transient high-entropy transition points. Some locally high-entropy tokens initiate a branch whose subsequent continuation rapidly becomes low-entropy. CL-EGRSD replaces instantaneous entropy with the minimum entropy over a short causal future window, separating sustained high-entropy spans from transient strategy-shift positions.

On Qwen3-4B and Qwen3-8B, the resulting confidence-gated update advances the accuracy--length frontier among the compared trainable methods (Figure~\ref{fig:teaser}). Ablations show that moderate entropy attenuation gives the most stable performance, and that lookahead helps most on the larger model, where pivots are easier to exploit.

Our contributions are:
\begin{itemize}[leftmargin=10pt, topsep=0pt, itemsep=1pt, partopsep=1pt, parsep=1pt]
    \item We identify teacher predictive entropy as a missing third signal in on-policy self-distillation, complementing outcome-reward direction and teacher--student magnitude.
    \item We instantiate this signal as \textbf{EGRSD}: a minimal extension of direction-aware self-distillation that gates the token update by the privileged teacher's entropy with a non-zero lower bound on every token weight.
    \item We extend EGRSD to \textbf{CL-EGRSD}, a causal-lookahead variant that preserves transient high-entropy \emph{pivot} tokens whose uncertainty resolves within a short future window, and validate both methods with main results, mechanism diagnostics, and ablations on Qwen3-4B and Qwen3-8B.
\end{itemize}

\section{Related work}
\label{sec:related}

We focus on the two lines of work most directly relevant to EGRSD. Additional context on long-form reasoning compression, RLVR-style token-level credit assignment, and a full method-by-method positioning comparison is deferred to Appendix~\ref{app:related}.

\paragraph{On-policy distillation and privileged self-distillation.}
On-policy distillation addresses the train--test mismatch of offline distillation by training on trajectories sampled from the student while a teacher provides dense token-level feedback~\citep{song2026opdsurvey}. OPSD adapts this idea to self-distillation: the same model acts as student and privileged teacher, with the teacher conditioned on additional information such as reference solutions~\citep{opsd}. This design removes the need for a separate large teacher and gives dense supervision on the student's own rollouts. Recent work further explores competence-aware weighting~\citep{xu2026paced}, consensus gating under privileged context~\citep{stein2026gates}, and compression-oriented variants such as CRISP~\citep{opsdc}. A common blind spot remains: existing OPSD-style objectives provide dense token-level targets but do not explicitly account for the concentration of the teacher distribution at each token. Concurrent work on direction-aware self-distillation (RLSD~\citep{yang2026rlsd}) couples outcome-reward direction with a teacher--student likelihood-ratio magnitude but leaves teacher confidence unused. EGRSD reuses the same coupling and additionally weights each token by the privileged teacher's predictive entropy.

\paragraph{Uncertainty, selectivity, and teacher confidence.}
Not all dense supervision is equally useful. Selective instruction-tuning and process supervision show that fine-grained filtering can improve how models learn from completions~\citep{li2023selective,lightman2023let}, and a complementary analysis~\citep{kim2026epistemic} finds that high-entropy teacher tokens carry uncertainty signaling that should be preserved rather than flattened. This finding motivates our non-zero floor on $\omega_{i,t}$. Concurrent with our work, SSD~\citep{zhang2026ssd} observes that code generation interleaves \emph{lock} positions (unambiguous continuations) with \emph{fork} positions (multiple plausible continuations), and reshapes token distributions differently at the two position types by temperature-truncated sampling. Entropy-aware on-policy distillation~\citep{jin2026entropyaware} instead mixes reverse and forward KL terms to alter the distillation geometry. EGRSD uses teacher entropy differently from both: it is a multiplicative confidence gate on the token-level RLSD update, making the lock/fork heterogeneity explicit without switching divergence objectives or training a separate uncertainty estimator. CL-EGRSD additionally rescues transient high-entropy \emph{pivot} positions whose continuation rapidly becomes confident, a regime not covered by SSD's two-category formulation.

\section{Background}
\label{sec:background}

We summarize the two objectives EGRSD builds on. Derivations, stop-gradient rationale, and advantage-whitening details are deferred to Appendix~\ref{app:background_full}.

\paragraph{Notation.}
Let $y_i=(y_{i,1},\ldots,y_{i,T_i})$ denote the $i$-th on-policy rollout sampled from the student policy $p_\theta$ on a problem $x$, with rollout positions $\mathcal{C}_i$ and mask $m_{i,t}=\mathbb{1}[t\in\mathcal{C}_i]$. The \emph{teacher} $p_T$ is the initial policy held fixed throughout training, conditioned on the privileged context $(x,s^\star)$ where $s^\star$ is a reference solution. The \emph{student} $p_S=p_\theta$ is initialized from the same weights, is the only trainable component, and is conditioned only on $(x)$.

\paragraph{On-policy self-distillation (OPSD).}
OPSD~\citep{opsd} aligns the student with the teacher at every rollout position:
\begin{equation}
\label{eq:opsd_loss}
    \mathcal{L}_\mathrm{OPSD}
    = \sum_{i,t:\,t\in\mathcal{C}_i}
    \mathrm{KL}\!\big(p_T(\cdot \mid x,s^\star,y_{i,<t})
    \,\|\, p_\theta(\cdot \mid x,y_{i,<t})\big),
\end{equation}
and CRISP~\citep{opsdc} adapts the same per-token alignment framework toward compression. Neither weights tokens by teacher confidence.

\paragraph{Direction-aware self-distillation (RLSD).}
RLSD~\citep{yang2026rlsd} replaces the teacher-driven update direction with an outcome-reward one. Each rollout receives a length-shaped reward $r_i = \mathbb{1}[y_i\text{ correct}]\cdot(1+\beta_L(1-|y_i|/L_\mathrm{max}))$ that is whitened into a sequence-level advantage $A_i$ with direction $D_i=\mathrm{sign}(A_i)$, shared across tokens in a rollout. The teacher--student log-ratio $\delta_{i,t}=\log p_T(y_{i,t}\mid x,s^\star,y_{i,<t})-\log p_S(y_{i,t}\mid x,y_{i,<t})$ is clipped into a multiplicative magnitude
\begin{equation}
\label{eq:rlsd_weight}
    w_{i,t} = \mathrm{clip}\!\big(\exp(D_i\,\delta_{i,t}),\,1-\varepsilon,\,1+\varepsilon\big),
    \qquad \varepsilon = 0.2,
\end{equation}
with stop-gradient on $p_S$ so that only $\log p_\theta$ in the final loss receives gradient. The RLSD loss is
\begin{equation}
\label{eq:rlsd_loss}
    \mathcal{L}_\mathrm{RLSD}
    = -\tfrac{1}{\sum_{i,t} m_{i,t}}
    \sum_{i,t} m_{i,t}\, A_i\, w_{i,t}\,
    \log p_\theta(y_{i,t}\mid x,y_{i,<t}).
\end{equation}
RLSD therefore couples reward-driven direction with a teacher--student likelihood ratio but treats every token under that ratio as equally informative. EGRSD modulates the same coupling by the teacher's predictive entropy.

\section{Method}
\label{sec:method}

\subsection{Motivation: a remaining confidence gap}
\label{sec:motivation}
\begin{figure*}[!t]
\centering
\includegraphics[width=\linewidth]{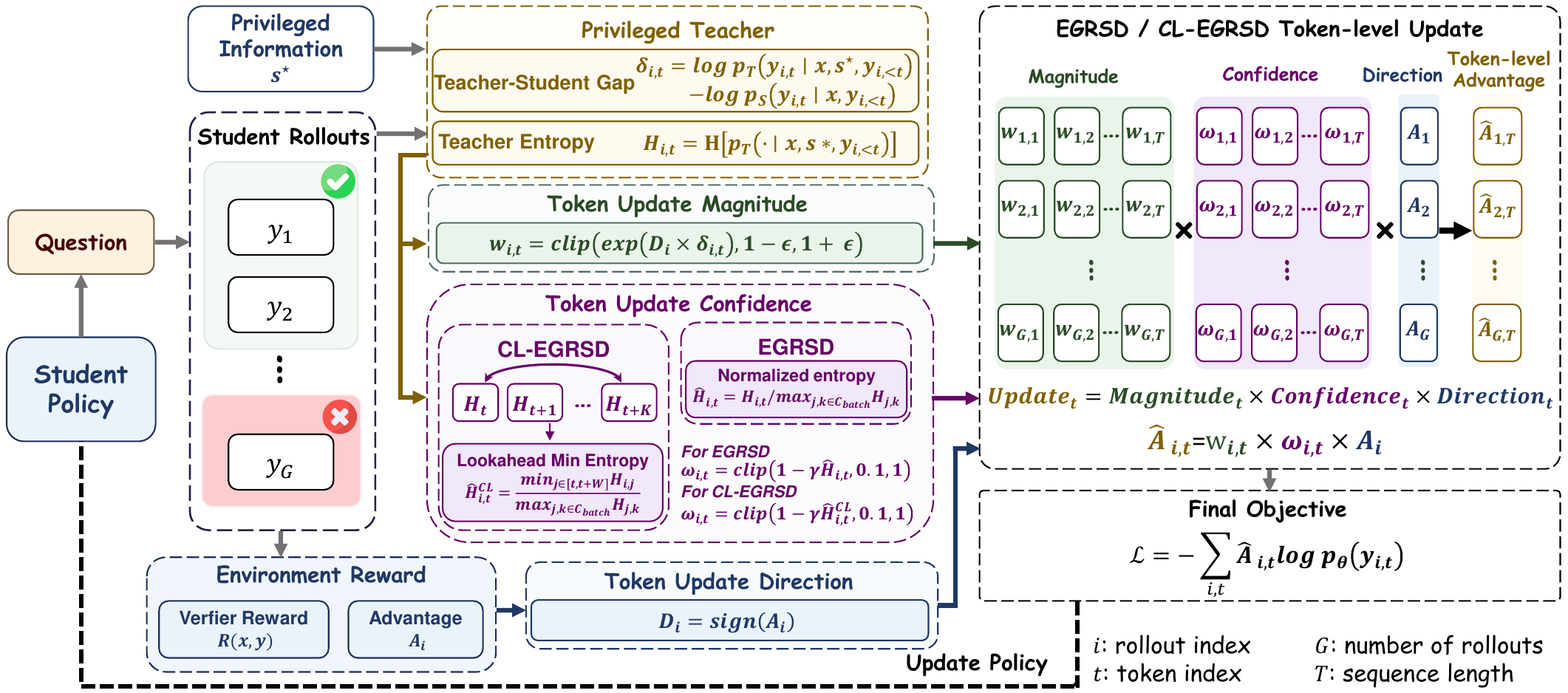}
\caption{Overview of the proposed method. The token-level update multiplies three signals: \emph{direction} from the sequence-level outcome-reward advantage $D_i=\mathrm{sign}(A_i)$ (shared across a rollout), \emph{magnitude} $w_{i,t}$ from the teacher--student likelihood ratio, and \emph{confidence} $\omega_{i,t}$ from a monotone gate of the privileged teacher's predictive entropy. CL-EGRSD replaces the instantaneous entropy inside the gate with a short-horizon causal-lookahead minimum to preserve transient pivot positions.}
\label{fig:method_schematic}
\vspace{-10pt}
\end{figure*}

The teacher--student ratio $w_{i,t}$ says whether the privileged view increases or decreases support for the sampled token. It does not say whether the privileged teacher distribution is concentrated. We use teacher confidence in an operational sense: the concentration of the privileged teacher's predictive distribution, measured through token-level entropy. Because the teacher is conditioned on augmented context, this confidence signal serves as a practical proxy for the reliability of token-level supervision without attributing cognitive states to the model. Direction-aware self-distillation can still assign large token-level magnitude to positions where the teacher distribution is diffuse. We therefore view token-level self-distillation as a three-signal problem. Table~\ref{tab:three_signals} summarizes how existing on-policy self-distillation objectives populate this picture. EGRSD targets the remaining confidence gap by introducing the teacher-entropy gate.

\begin{table*}[!h]
\centering
\small
\caption{Signals used by families of on-policy self-distillation objectives. EGRSD augments the direction--magnitude decomposition with an explicit teacher-confidence signal.}
\label{tab:three_signals}
\begin{tabular}{l c c c}
\toprule
\textbf{Method} & \textbf{Direction} & \textbf{Magnitude} & \textbf{Confidence} \\
\midrule
GRPO & outcome reward & uniform & none \\
OPSD & teacher & teacher distribution & none \\
RLSD & outcome reward & teacher--student ratio & none \\
\textbf{EGRSD} & outcome reward & teacher--student ratio & teacher entropy \\
\bottomrule
\end{tabular}
\end{table*}

Before formalizing this gate, it helps to fix intuition about what the teacher-entropy signal captures. We distinguish three token regimes by predictive-distribution concentration, illustrated on a real reasoning example in Figure~\ref{fig:token_entropy_motivation}:
\begin{itemize}[leftmargin=10pt, topsep=0pt, itemsep=1pt, partopsep=1pt, parsep=1pt]
    \item \textbf{Lock.} Low instantaneous entropy, with the teacher sharply peaked. These positions correspond to deterministic computation (arithmetic, unit manipulation, equation continuation) and provide reliable token-level supervision.
    \item \textbf{Fork.} High instantaneous entropy that \emph{stays} high over the following tokens, with the teacher diffuse and multiple continuations plausible. Sustained branching of this kind can let the teacher--student likelihood ratio alone overstate the evidence.
    \item \textbf{Pivot.} High instantaneous entropy that \emph{resolves} to low entropy within a short future window, where the teacher is locally uncertain but immediately recommits. Such transitions, often cued by discourse markers, are the target of the causal-lookahead variant below.
\end{itemize}

\begin{figure}[t]
\centering
\includegraphics[trim=13.6bp 8.1bp 15.3bp 4.7bp,clip,width=\linewidth]{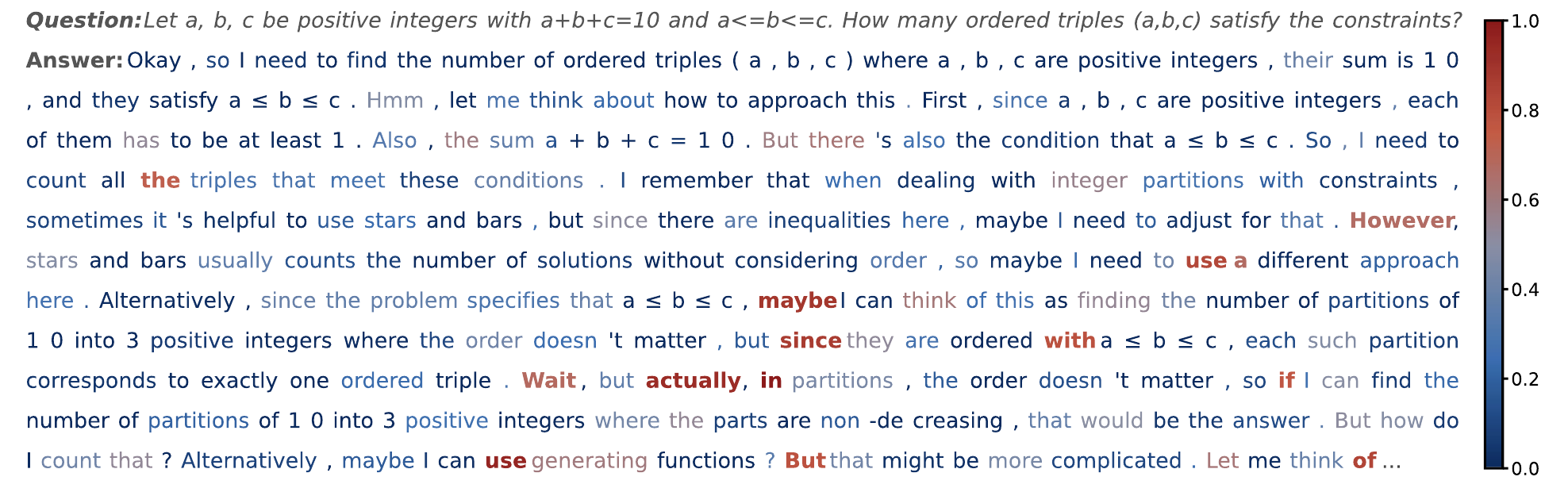}
\caption{Per-token predictive entropy on a representative reasoning trace (Qwen3-4B). Tokens are color-coded by the teacher's entropy and bolded at the top ${\sim}35\%$ of positions. High-entropy tokens concentrate at discourse transitions and strategy shifts. Low entropy marks routine computation. Appendix~\ref{app:qualitative_pivots} gives an annotated version with $\widehat{H}_{i,t}$ and $\widehat{H}_{i,t}^{\mathrm{CL}}$.}
\label{fig:token_entropy_motivation}
\end{figure}

Uniformly trusting all tokens over-weights fork evidence, whereas uniformly suppressing all high-entropy tokens destroys pivot evidence. EGRSD~(\S\ref{sec:egrsd}) handles the first failure mode by attenuating high-entropy positions while retaining a non-zero floor, and CL-EGRSD~(\S\ref{sec:clegrsd}) handles the second by swapping in a short-horizon future entropy so that pivots are selectively restored.

Figure~\ref{fig:method_schematic} summarizes the resulting token-level update. The method is intentionally minimal: it changes neither the rollout distribution nor the reward definition, and adds only a multiplicative entropy-based confidence factor on top of the existing direction-aware update.

\subsection{Entropy-guided reinforced self-distillation}
\label{sec:egrsd}

For each rollout $i$ and position $t\in\mathcal{C}_i$ we compute the privileged-teacher entropy
\begin{equation}
    H_{i,t} =
    -\sum_{v\in\mathcal{V}}
    p_T(v\mid x,s^\star,y_{i,<t})
    \log p_T(v\mid x,s^\star,y_{i,<t}).
\end{equation}
Let $\mathcal{C}_\mathrm{batch}=\bigcup_i \{(i,t):t\in\mathcal{C}_i\}$ denote all rollout positions in the current minibatch. We normalize by the batch-global maximum
\begin{equation}
    \widehat{H}_{i,t} =
    \frac{H_{i,t}}{\displaystyle\max_{(j,k)\in\mathcal{C}_\mathrm{batch}} H_{j,k}}.
\end{equation}
In the implementation we additionally lower-bound this denominator by $1$ nat to avoid numerical instability on low-entropy minibatches. The confidence gate is
\begin{equation}
    \omega_{i,t} =
    \mathrm{clip}\!\left(1-\gamma\,\widehat{H}_{i,t},\,0.1,\,1\right).
    \label{eq:entropy_weight_improved}
\end{equation}
The coefficient $\gamma$ controls the sensitivity of the gate to teacher entropy. Larger $\gamma$ applies stronger attenuation to high-entropy positions. Setting $\gamma=0$ recovers the direction-aware baseline without the entropy gate. The lower bound prevents high-entropy positions from being removed entirely, since some high-entropy positions reflect valid branching in the predictive distribution rather than spurious supervision. This choice is supported by the empirical finding of~\citet{kim2026epistemic} that hard suppression of high-entropy teacher tokens degrades reasoning.

\paragraph{Geometric interpretation of the linear gate.}
The form $\omega_{i,t} = 1 - \gamma\,\widehat{H}_{i,t}$ is the endpoint chord (the secant at $\widehat{H}\in\{0,1\}$) of the worst-case shrinkage bound $\omega^\star_{a_0}(\widehat{H})=1/(1+a_0\widehat{H})$ derived from the linear noise-to-signal proxy $\sigma^2/\mu^2\le a_0\widehat{H}$. Therefore $\gamma$ admits a geometric reading as the noise-to-signal ratio at maximum teacher entropy, $\gamma=\mathrm{NSR}_{\max}/(1+\mathrm{NSR}_{\max})$ (Appendix~\ref{app:linear_gate}).

EGRSD combines direction, magnitude, and confidence:
\begin{equation}
    \widehat{A}_{i,t} = A_i \cdot w_{i,t} \cdot \omega_{i,t},
    \qquad
    \mathcal{L}_\mathrm{EGRSD}
    =
    -\frac{1}{\sum_{i,t} m_{i,t}}
    \sum_{i,t} m_{i,t}\,
    \widehat{A}_{i,t}\,
    \log p_\theta(y_{i,t}\mid x,y_{i,<t}).
    \label{eq:egrsd_loss_improved}
\end{equation}
$\widehat{A}_{i,t}$ is a stop-gradient constant. The gradient flows only through $\log p_\theta$.

\subsection{Causal-lookahead EGRSD}
\label{sec:clegrsd}

Instantaneous entropy can conflate sustained high-entropy spans with transient transition points. A token may have high entropy because the sequence distribution branches locally, while the subsequent continuation rapidly becomes low entropy. CL-EGRSD therefore replaces $H_{i,t}$ with a causal future-window minimum before applying the same batch-global normalization:
\begin{equation}
    \widehat{H}_{i,t}^{\mathrm{CL}}
    =
    \frac{\displaystyle\min_{j\in[t,\,\min(t+W,T_i)]} H_{i,j}}{\displaystyle\max_{(j,k)\in\mathcal{C}_\mathrm{batch}} H_{j,k}},
    \qquad
    \omega_{i,t}^{\mathrm{CL}}
    =
    \mathrm{clip}\!\left(1-\gamma\,\widehat{H}_{i,t}^{\mathrm{CL}},\,0.1,\,1\right).
\end{equation}
The lookahead window is truncated at the sequence end via $\min(t+W,T_i)$. We take the minimum of the raw entropies within the window and reuse the same batch-global denominator. This preserves transient high-entropy positions that resolve within $W$ tokens and continues to attenuate sustained high-entropy spans.

\paragraph{Why minimum?}
Among causal smoothing filters $\phi(H_{i,t},\ldots,H_{i,t+W})$ that are per-argument monotone, conservative ($\phi\le H_{i,t}$), idempotent ($\phi(c,\ldots,c)=c$), and causal, any member satisfies $\phi\ge\min_{j\in[t,t+W]}H_{i,j}$, so the minimum is the extremal choice that maximizes weight recovery at pivot positions (those with high current entropy but low lookahead entropy), while leaving sustained high-entropy (fork) positions nearly untouched. The formal statement and proof are in Appendix~\ref{app:min_pool}.

\section{Experiments}
\label{sec:experiments}

\subsection{Setup}
\label{sec:setup}

\paragraph{Models and data.}
We evaluate Qwen3-4B and Qwen3-8B with thinking mode enabled~\citep{qwen2025qwen3}. Teacher and student share the backbone weights: the teacher's weights are not updated through LoRA, and the teacher is conditioned on the privileged reference solution $s^\star$, while the student uses the LoRA-adapted weights without $s^\star$. Our training data configuration matches that of OPSD~\citep{opsd}. The teacher receives $(x,s^\star)$ and the student receives $(x)$.

\paragraph{Training.}
We train all methods with AdamW~\citep{loshchilov2019decoupled} at a learning rate of $5\times10^{-6}$ in BF16 mixed precision, under an identical update budget. LoRA~\citep{hu2022lora} ($r{=}64$, $\alpha{=}128$, no dropout) is applied to each Transformer layer's attention projections (\texttt{q\_proj}, \texttt{k\_proj}, \texttt{v\_proj}, \texttt{o\_proj}) and MLP projections (\texttt{gate\_proj}, \texttt{up\_proj}, \texttt{down\_proj}). On-policy rollouts are generated with vLLM~\citep{kwon2023vllm} in colocate mode at temperature $1.1$, top-$p{=}0.95$, top-$k{=}20$, and a per-step completion cap of $1{,}024$ tokens. The teacher is kept frozen throughout training. Section~\ref{sec:teacher_update} ablates EMA and hard-copy update schedules and shows both underperform the frozen choice. Training dynamics (pre-clip gradient norm) are reported in Appendix~\ref{app:dynamics}. Additional training details are deferred to Appendix~\ref{app:setup}.

\paragraph{Evaluation.}
We evaluate on AIME 2024, AIME 2025, HMMT 2025, MATH-500, Minerva-Math, and GSM8K. For inference, we employ vLLM with the official Qwen3 chat template and \texttt{enable\_thinking=True}. We generate $K=4$ samples per prompt using a temperature of $1.0$, top-$p=0.95$, and a maximum of $32{,}768$ new tokens. To quantify AIME24 variance, results on that benchmark are averaged over five independent runs. The remaining benchmarks use a single run. We report avg@$K$ for each benchmark and define Avg. as the macro-average of avg@$K$ across the six benchmarks. Generation length is averaged across benchmarks and samples.

\paragraph{Compared methods.}
We compare against the no-train baseline, SFT, GRPO~\citep{shao2024deepseekmath}, OPSD~\citep{opsd}, and CRISP~\citep{opsdc}. %All trainable methods use the same training budget.

\subsection{Main comparison}
\label{sec:main_results}

Tables~\ref{tab:main_4b_improved} and~\ref{tab:main_8b_improved} show the main comparison, with all trainable methods sharing the same 100-step training budget. Additional variants and sweeps appear below and in the appendix. EGRSD and CL-EGRSD outperform all trainable baselines (SFT, GRPO~\citep{shao2024deepseekmath}, OPSD~\citep{opsd}, and CRISP~\citep{opsdc}) on both Qwen3-4B and Qwen3-8B. On 4B, EGRSD Pareto-dominates OPSD in both accuracy and length (Table~\ref{tab:efficiency_full}). On 8B, CL-EGRSD gives the highest observed Avg. in this comparison, with the largest per-benchmark gain on HMMT25 ($+7.40$ over the no-train baseline), at a generation length ($12{,}232$ tokens) comparable to OPSD's ($12{,}097$). All trainable baselines are Pareto-dominated in the accuracy--length plane, as visualized in Figure~\ref{fig:teaser}. The primary takeaway is that confidence-aware weighting advances the accuracy--length frontier. % in Figure~\ref{fig:token_efficiency_8b}(Appendix~\ref{app:efficiency_8b}). The primary takeaway is that confidence-aware weighting advances the accuracy--length frontier.

\begin{table*}[t]
\centering
\small
\caption{Results on \textbf{Qwen3-4B}. We express all evaluative metrics as percentages. The best and runner-up results are highlighted with \textbf{bold} and \underline{underline}, respectively.}
\label{tab:main_4b_improved}
\resizebox{\textwidth}{!}{
\begin{tabular}{l c c c c c c c}
\toprule
\textbf{Method} & \textbf{AIME24} & \textbf{AIME25} & \textbf{HMMT25} & \textbf{MATH500} & \textbf{Minerva} & \textbf{GSM8K} & \textbf{Avg.} \\
\midrule
Baseline (no-train)   & 73.02 & 65.00 & 41.56 & 86.93 & \underline{33.09} & 93.96 & 65.59 \\
SFT                   & 58.61 & 42.50 & 30.83 & 83.65 & 26.47 & 87.19 & 54.88 \\
GRPO~\citep{shao2024deepseekmath} & 73.06 & 63.33 & \textbf{45.00} & \underline{87.15} & 32.90 & 93.95 & 65.90 \\
OPSD~\citep{opsd}     & 74.48 & \underline{67.92} & 43.12 & 87.02 & \textbf{33.23} & 93.89 & 66.61 \\
CRISP~\citep{opsdc}   & 74.06 & 63.33 & 40.94 & 86.94 & 32.84 & \underline{94.01} & 65.35 \\
\midrule
\rowcolor{ourcolor}
\textbf{EGRSD} (ours)   & \textbf{77.92} & \textbf{68.75} & 42.92 & \textbf{87.17} & 32.81 & 93.88 & \textbf{67.24} \\
\rowcolor{ourcolor}
\textbf{CL-EGRSD} (ours) & \underline{77.08} & 67.08 & \underline{44.58} & 86.92 & 32.90 & \textbf{94.26} & \underline{67.14} \\
\bottomrule
\end{tabular}
}
\end{table*}

\begin{table*}[t]
\centering
\small
\caption{Results on \textbf{Qwen3-8B}. We express all evaluative metrics as percentages. The best and runner-up results are highlighted with \textbf{bold} and \underline{underline}, respectively.}
\label{tab:main_8b_improved}
\resizebox{\textwidth}{!}{
\begin{tabular}{l c c c c c c c}
\toprule
\textbf{Method} & \textbf{AIME24} & \textbf{AIME25} & \textbf{HMMT25} & \textbf{MATH500} & \textbf{Minerva} & \textbf{GSM8K} & \textbf{Avg.} \\
\midrule
Baseline (no-train)   & 75.11 & 67.71 & 45.10 & 86.92 & 34.32 & 94.53 & 67.28 \\
SFT                   & 73.61 & 66.67 & \underline{46.67} & 86.60 & 33.92 & 94.45 & 66.99 \\
GRPO~\citep{shao2024deepseekmath} & 75.28 & \textbf{71.67} & 41.67 & 86.60 & 34.65 & 94.58 & 67.41 \\
OPSD~\citep{opsd}     & 73.75 & 68.75 & 45.84 & \underline{87.11} & \underline{34.94} & \textbf{94.68} & 67.51 \\
CRISP~\citep{opsdc}   & 76.04 & 67.29 & 42.92 & 86.93 & 34.47 & 94.57 & 67.04 \\
\midrule
\rowcolor{ourcolor}
\textbf{EGRSD} (ours)   & \underline{76.67} & \underline{70.00} & 45.42 & \textbf{87.20} & \textbf{35.48} & \underline{94.61} & \underline{68.23} \\
\rowcolor{ourcolor}
\textbf{CL-EGRSD} (ours) & \textbf{77.22} & \underline{70.00} & \textbf{52.50} & 86.95 & 34.83 & 94.43 & \textbf{69.32} \\
\bottomrule

\end{tabular}
}
\end{table*}

\subsection{Ablations on entropy strength and lookahead}
\label{sec:ablations}

Table~\ref{tab:gamma_improved} isolates the entropy coefficient across four training snapshots on Qwen3-8B. Moderate attenuation is the most stable regime: $\gamma{=}0.1$ attains the highest single-snapshot Avg. of $68.32$ and $\gamma{=}0.3$ peaks at $68.23$ with a smoother trajectory, while very weak ($\gamma{=}0.0$) and strong ($\gamma{\geq}0.7$) settings trail at most snapshots. Moderate attenuation balances peak accuracy against snapshot-to-snapshot variance, consistent with retaining a nonzero lower bound for high-entropy positions.

Token efficiency $\mathrm{Eff}=\mathrm{Acc}(\%)/(\mathrm{AvgLen}/1000)$ on Qwen3-4B (Table~\ref{tab:efficiency_full}) gives a complementary picture. EGRSD ($\mathrm{Eff}{=}6.08$) and CL-EGRSD ($\mathrm{Eff}{=}6.06$) are the only trainable methods that improve over SFT ($\mathrm{Eff}{=}6.05$). GRPO, OPSD, and CRISP all produce longer generations without commensurate accuracy gains and fall to $\mathrm{Eff}{\leq}5.86$. SFT's nominal $6.05$ comes from a $\sim$1.9K-token compression paired with a ten-point accuracy drop rather than a true efficiency improvement.

\begin{table}[t]
\centering
\small
\begin{minipage}[t]{0.4\textwidth}
\centering
\caption{EGRSD $\gamma$ sweep on Qwen3-8B. Columns S1--S4 are four training snapshots, and $\gamma=0$ removes entropy weighting.}
\label{tab:gamma_improved}
\scalebox{0.955}{%
\begin{tabular}{c c c c c}
\toprule
$\gamma$ & \textbf{S1} & \textbf{S2} & \textbf{S3} & \textbf{S4} \\
\midrule
0.0   & 67.16 & \textbf{67.97} & 67.64 & 66.81 \\
% \rowcolor{bestcolor}
0.1   & 66.82 & 66.92 & 67.67 & \textbf{68.32} \\
0.2   & 67.00 & 66.89 & 67.39 & \textbf{67.79} \\
0.3   & 66.99 & 67.47 & 66.53 & \textbf{68.23} \\
0.4   & 67.53 & \textbf{67.72} & 67.03 & 66.82 \\
0.6   & 66.87  & 67.33 & 67.70 &  \textbf{68.06}\\
0.7   & 67.00 & 66.88 & 67.70 & \textbf{68.00} \\
0.8   & 66.85 & 66.85 & \textbf{67.78} & 67.23 \\
0.9   & 66.58 & 66.76 & 67.43 & \textbf{67.83} \\
\bottomrule
\end{tabular}}
\end{minipage}\hfill
\begin{minipage}[t]{0.56\textwidth}
\centering
\caption{Token efficiency on Qwen3-4B, defined as $\mathrm{Eff}=\mathrm{Acc}/(\mathrm{AvgLen}/1000)$ with accuracy in percent and length in thousands of tokens. EGRSD and CL-EGRSD are the only trainable methods that improve over SFT.}
\label{tab:efficiency_full}
\scalebox{1.0437}{%
\begin{tabular}{l c c c}
\toprule
\textbf{Method} & \textbf{Acc} & \textbf{AvgLen} & \textbf{Eff} \\
\midrule
Baseline            & 65.59 & 11{,}008 & 5.96 \\
SFT                 & 54.88 & \phantom{0}9{,}070 & 6.05 \\
GRPO~\citep{shao2024deepseekmath}  & 65.90 & 11{,}809 & 5.58 \\
OPSD~\citep{opsd}   & 66.61 & 11{,}787 & 5.65 \\
CRISP~\citep{opsdc} & 65.35 & 11{,}157 & 5.86 \\
\rowcolor{bestcolor}
\textbf{EGRSD (ours)} & {67.24} & {11{,}064} & \textbf{6.08} \\
\rowcolor{bestcolor}
\textbf{CL-EGRSD (ours)}     & 67.14 & 11{,}085 & \textbf{6.06} \\
\bottomrule
\end{tabular}}
\end{minipage}
\end{table}

\begin{table}[t]
\centering
\small
\caption{Lookahead ablation on Qwen3-8B. We express all evaluative metrics as percentages. The best and runner-up results are highlighted with \textbf{bold} and \underline{underline}, respectively.}
\label{tab:window_improved}
\resizebox{\textwidth}{!}{
\begin{tabular}{l c c c c c c c c c}
\toprule
\textbf{Variant} & \textbf{AIME24} & \textbf{AIME25} & \textbf{HMMT25} & \textbf{MATH500} & \textbf{Minerva} & \textbf{GSM8K} & \textbf{Avg.} & \textbf{AvgLen} \\
\midrule
EGRSD ($W{=}0$)      & \underline{76.67} & 70.00 & 45.42 & \underline{87.20} & 35.48 & \underline{94.61} & 68.23 & 11{,}836 \\
CL-EGRSD ($W{=}1$)   & 75.56 & \underline{71.67} & 46.67 & 86.95 & \underline{36.12} & \textbf{94.66} & \underline{68.60} & 11{,}962 \\
CL-EGRSD ($W{=}3$)   & 75.00 & 66.67 & \underline{50.00} & \textbf{87.50} & 34.74 & 94.41 & 68.05 & 12{,}005 \\
CL-EGRSD ($W{=}5$)   & \textbf{77.22} & 70.00 & \textbf{52.50} & 86.95 & 34.83 & 94.43 & \textbf{69.32} & 12{,}232 \\
CL-EGRSD ($W{=}7$)   & 74.72 & \textbf{74.16} & 41.25 & 87.15 & \textbf{36.31} & 94.51 & 68.02 & 12{,}231 \\
\bottomrule
\vspace{-0pt}

\end{tabular}
}
\end{table}

Table~\ref{tab:window_improved} reports the lookahead sweep on Qwen3-8B. $W=5$ gives the strongest Avg., outperforming both shorter and longer windows. A focused sweep under stronger suppression (Appendix~\ref{app:clgrid}) shows larger lookahead gains, motivating joint tuning of $\gamma$ and $W$.

\subsection{Teacher update schedule ablation}
\label{sec:teacher_update}

Our method uses a \emph{frozen} teacher throughout training: $p_T$ is the base model with LoRA adapter disabled (Appendix~\ref{app:background_full}) and is never updated. A natural question is whether allowing the teacher to track the student could help. We therefore ablate two families of teacher-update rules on Qwen3-8B with EGRSD at $\gamma=0.3$, keeping the rest of the training protocol identical to the main result. (i) \textbf{Exponential moving average (EMA)}: teacher weights are an EMA of the student, $p_T\leftarrow\alpha\,p_T+(1-\alpha)\,p_S$, with $\alpha=0.99$ (fast tracking, $\sim$100-step lag). (ii) \textbf{Hard copy}: every $K$ steps the teacher is replaced by a hard copy of the current student. We test $K\in\{20,50\}$.

As shown in Table~\ref{tab:teacher_update}, all three online-update schedules underperform the frozen teacher. %EMA ($\alpha{=}0.99$) loses $0.40$ Avg.; hard copy further loses $1.11$ at $K{=}20$ and $1.59$ at $K{=}50$, so the discrete replacements also underperform EMA at comparable update frequency, which we attribute to the discontinuity they introduce into $p_T$.
Mechanistically, when the teacher tracks the student, the log-ratio $\delta_{i,t}=\log p_T-\log p_S$ collapses towards zero, the magnitude $w_{i,t}=\mathrm{clip}(\exp(D_i\delta_{i,t}),1-\varepsilon,1+\varepsilon)$ collapses towards $1$, and EGRSD degrades towards a plain outcome-reward objective with an entropy gate. The benefit of conditioning the teacher on privileged context $(x,s^\star)$ is eroded. This connects to the epistemic-verbalization finding of \citet{kim2026epistemic}, who show that a teacher whose predictive distribution has been compressed (in their case by folding the privileged answer in too aggressively) produces worse supervision because epistemic uncertainty markers are suppressed. Updating the teacher towards the current student is a different but related pathway to the same failure mode: the teacher drifts away from its original calibrated distribution and loses the outside-view reference that makes this three-signal decomposition informative. A frozen $p_T$ avoids both over-confidence and this drift-based calibration loss. We therefore adopt it as the default.

\subsection{Mechanism analysis}
\label{sec:mechanism}

\paragraph{Entropy diagnoses unreliable evidence.}
\label{sec:entropy_diagnostic}

\begin{table}[t]
\centering
\small
\caption{Teacher-update ablation on Qwen3-8B with EGRSD ($\gamma=0.3$).}
\label{tab:teacher_update}
\resizebox{\textwidth}{!}{
\begin{tabular}{l c c c c c c c}
\toprule
\textbf{Teacher update} & \textbf{AIME24} & \textbf{AIME25} & \textbf{HMMT25} & \textbf{MATH500} & \textbf{Minerva} & \textbf{GSM8K} & \textbf{Avg.} \\
\midrule
\rowcolor{bestcolor}
Frozen (ours) & \textbf{76.67} & \textbf{70.00} & \textbf{45.42} & 87.20 & \textbf{35.48} & 94.61 & \textbf{68.23} \\
EMA ($\alpha{=}0.99$) & 73.89 & 70.83 & 45.00 & 87.25 & 35.29 & \textbf{94.73} & 67.83 \\
Hard copy ($K{=}20$) & 73.89 & 67.50 & 45.00 & \textbf{87.45} & 34.19 & 94.67 & 67.12 \\
Hard copy ($K{=}50$) & 74.44 & 67.50 & 40.83 & 87.30 & 35.20 & 94.56 & 66.64 \\
\bottomrule
\end{tabular}
}
\end{table}

\begin{figure}[t]
\centering
\begin{subfigure}{0.49\linewidth}
\centering
\includegraphics[width=\linewidth]{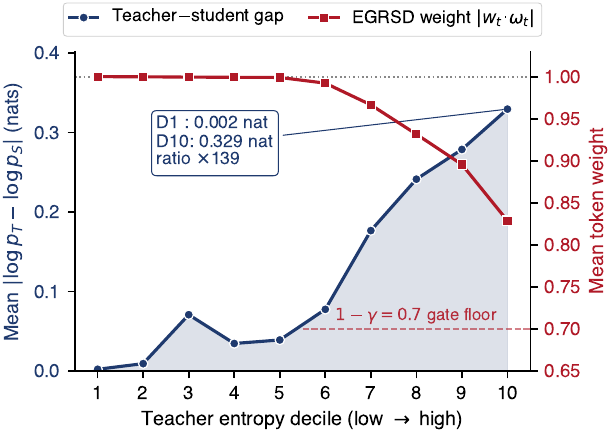}
\caption{Teacher--student evidence gap vs.\ EGRSD update weight per entropy decile.}
\label{fig:entropy_reliability}
\end{subfigure}\hfill
\begin{subfigure}{0.49\linewidth}
\centering
\includegraphics[width=\linewidth]{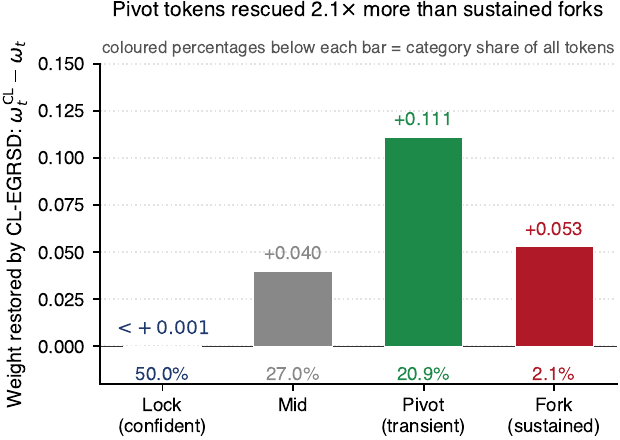}
\caption{Mean lookahead weight increment $\omega_{i,t}^{\mathrm{CL}}-\omega_{i,t}$ by regime.}
\label{fig:cl_pivot_diagnostic}
\end{subfigure}
\caption{Mechanism diagnostics on 5.5M held-out tokens (1{,}688 completions from AIME24/25, HMMT25, Minerva). \textbf{Left:} the mean teacher--student log-prob gap grows from $0.00237$ nats (decile 1) to $0.329$ nats (decile 10), a $\times 139$ spread, while EGRSD attenuates the mean update weight in high-entropy deciles and stays above the $1-\gamma=0.7$ gate floor. \textbf{Right:} CL-EGRSD restores pivot weight by $+0.111$ on average versus $+0.053$ for sustained forks (a $2.1\times$ selectivity), while leaving lock tokens near the weight ceiling ($\gamma=0.3,W=5$).}
\label{fig:mechanism_combined}
\end{figure}

Figure~\ref{fig:entropy_reliability} tests the central mechanism directly. High-entropy positions coincide with substantially less concentrated teacher--student evidence: the top entropy decile has roughly 140$\times$ the log-prob gap of the bottom decile. EGRSD therefore does not treat entropy as a length penalty. It uses entropy as a confidence signal to calibrate which token-level evidence should receive a strong learning signal. This empirical pattern matches the finding of \citet{kim2026epistemic} that high-entropy teacher tokens carry uncertainty that should be preserved rather than suppressed: the high-entropy decile is precisely where teacher evidence becomes least reliable, which is why the confidence gate retains a nonzero floor there rather than zeroing out the signal.

\paragraph{Lookahead targets transient pivots.}
\label{sec:pivot_diagnostic}

To quantify how CL-EGRSD reshapes token weight across the three regimes introduced in \S\ref{sec:motivation}, we use operational thresholds on normalized teacher entropy: a position $(i,t)$ is classified as \emph{lock} if $\widehat{H}_{i,t}\leq\tau_\mathrm{low}$, as \emph{fork} if $\widehat{H}_{i,t}\geq\tau_\mathrm{high}$ \emph{and} $\widehat{H}_{i,t}^{\mathrm{CL}}\geq\tau_\mathrm{high}$, and as \emph{pivot} if $\widehat{H}_{i,t}\geq\tau_\mathrm{high}$ but $\widehat{H}_{i,t}^{\mathrm{CL}}\leq\tau_\mathrm{low}$. All other positions are labeled \emph{mid}. EGRSD attenuates weight uniformly on all high-entropy positions through $\omega_{i,t}$. CL-EGRSD's causal-lookahead replacement instead selectively restores weight on pivots by swapping in a low-entropy future value.

Figure~\ref{fig:cl_pivot_diagnostic} explains why lookahead is useful but conditional. We categorize tokens by current entropy $\widehat{H}_{i,t}$ and five-token lookahead entropy $\widehat{H}_{i,t}^{\mathrm{CL}}$: \emph{lock} (low, low), \emph{fork} (high, still high), \emph{pivot} (high, resolves to low), and \emph{mid} (remainder). Pivots constitute $20.9\%$ of the analyzed tokens, and at $\gamma=0.3, W=5$ CL-EGRSD restores their weight by $+0.111$ on average versus $+0.053$ for sustained forks (a $2.1\times$ selectivity ratio). Lock restore is below $0.001$ because low-entropy positions are already near the weight ceiling. This matches the intended behavior: preserve transient transition points without undoing entropy suppression everywhere. On 4B, lookahead is less consistent, making EGRSD the cleaner default. On 8B, the model appears better able to exploit protected transitions.

\section{Conclusion}
\label{sec:conclusion}

Across Qwen3-4B and Qwen3-8B mathematical reasoning benchmarks with thinking mode enabled, an entropy-gated token-level update is the only intervention in our comparison that improves over the no-train baseline on the accuracy--length frontier. The gate has two forms: EGRSD attenuates high-entropy positions through the instantaneous teacher entropy, and CL-EGRSD swaps in the minimum entropy over a short causal future window to preserve transient pivot tokens. EGRSD Pareto-dominates OPSD on Qwen3-4B, and CL-EGRSD attains the highest Avg. on Qwen3-8B at a length comparable to OPSD.

\bibliographystyle{plainnat}
\bibliography{references}

\clearpage
\appendix
\startcontents[appendix]
\printcontents[appendix]{ }{0}{\section*{Appendix}}
% \newpage
% \section*{Appendix}
% Prefix appendix figures / tables / algorithms to avoid hyperref destination collision.
\renewcommand{\thefigure}{A\arabic{figure}}
\renewcommand{\thetable}{A\arabic{table}}
\renewcommand{\theHfigure}{appendix.\arabic{figure}}
\renewcommand{\theHtable}{appendix.\arabic{table}}
\setcounter{figure}{0}
\setcounter{table}{0}

\section{Extended related work}
\label{app:related}

This appendix complements the two paragraphs retained in the main Related Work section with additional context on (i) long-form reasoning efficiency, (ii) RLVR-style token-level credit assignment, and (iii) a method-by-method positioning summary.

\paragraph{Long-form reasoning and reasoning efficiency.}
Chain-of-thought prompting and recent reasoning-oriented models have shown that allocating more inference-time computation can substantially improve multi-step mathematical problem solving~\citep{wei2022chain,openai2024o1,deepseekr1,qwen2025qwen3}. This gain comes with a practical cost: reasoning models often emit long deliberations containing redundant checks, false starts, and repeated verification. A growing line of work therefore studies reasoning compression and efficient post-training, including reward-based length control, supervised fine-tuning on concise completions, training-free prompting, pruning, and latent or self-distillation-based compression. Chain of Draft reduces reasoning verbosity at inference time by prompting models to write concise intermediate drafts~\citep{xu2025chainofdraft}. CRISP~\citep{opsdc} is the closest compression-oriented baseline to our setting and uses iterative self-policy distillation to encourage concise reasoning. Our method is complementary in goal but different in mechanism: rather than imposing a prompt-based conciseness instruction or treating all teacher positions uniformly, EGRSD changes the token-level confidence weighting of the distillation signal by down-weighting high-entropy teacher positions.

\paragraph{RLVR and token-level credit assignment.}
Reinforcement learning with verifiable rewards has become a standard post-training approach for reasoning models because correctness can often be checked automatically~\citep{deepseekr1}. A widely used instantiation is Group Relative Policy Optimization (GRPO)~\citep{shao2024deepseekmath}, which estimates token-level advantages from sequence-level outcome rewards across a rollout group. Its main limitation is sparse credit assignment: a sequence-level reward or advantage is broadcast across all tokens in a long rollout, even though only a small subset of tokens may determine the outcome. Process-level supervision and process reward models address this by scoring intermediate reasoning steps, but they require additional annotation, modeling, or inference machinery~\citep{lightman2023let}. Concurrent work on direction-aware self-distillation (RLSD~\citep{yang2026rlsd}) casts self-distillation as a token-level credit-assignment tool: outcome rewards determine update direction, while the privileged teacher--student likelihood ratio modulates update magnitude. Related distillation-RL hybrids such as reinforcement-aware knowledge distillation also modify how teacher probabilities interact with policy optimization~\citep{zhang2026trrd}. EGRSD also adopts this direction--magnitude pair and adds a missing confidence signal: when the privileged teacher distribution is diffuse, the teacher--student ratio alone can overstate the confidence of the token-level evidence. Entropy gating reduces this effect without changing the reward-grounded update direction.

\paragraph{Positioning.}
Overall, EGRSD sits at the intersection of reasoning compression, on-policy self-distillation, and token-level credit assignment. Compared with OPSD~\citep{opsd}, it weights privileged teacher positions according to teacher confidence rather than uniformly. Against CRISP~\citep{opsdc}, the change is in token-level credit allocation instead of a conciseness instruction. Relative to RLSD~\citep{yang2026rlsd}, EGRSD shares the direction--magnitude decomposition and adds a teacher-entropy confidence signal. The closest contemporary point of contact is SSD~\citep{zhang2026ssd}, which implicitly reshapes token distributions at lock/fork positions. Our gate is an explicit multiplicative factor on the token-level RLSD update, and CL-EGRSD further introduces a pivot regime not covered by SSD's two-category view. Unlike process-supervision approaches~\citep{lightman2023let}, EGRSD obtains fine-grained confidence information without extra annotations.

\section{Background details}
\label{app:background_full}

This section elaborates the notation, stop-gradient rationale, and advantage-whitening procedure that the compressed main-text Background (\S\ref{sec:background}) omits for brevity.

\paragraph{Teacher and student as a shared PEFT instance.}
Teacher and student share the same backbone. The teacher's weights are not updated through LoRA and are conditioned on the privileged context $(x,s^\star)$, while the student's LoRA adapter is active and sees only $(x)$. The privileged context $s^\star$ is inserted via a two-turn chat template consisting of a reference reasoning segment, a transition prompt, and the student's on-policy rollout. Further template details are deferred to Appendix~\ref{app:setup}. Because $p_T$ and $p_S$ differ only in LoRA state and conditioning, the teacher--student log-ratio $\delta_{i,t}$ captures exactly the effect of the privileged context plus LoRA adaptation.

\paragraph{Stop-gradient rationale.}
When computing $\delta_{i,t}$ and the magnitude $w_{i,t}$ in Eq.~\ref{eq:rlsd_weight}, we apply \texttt{stop-gradient} to $p_S$. Without this, gradients would flow through $\log p_S$ inside $w_{i,t}$ and bypass the outcome reward: the network could reduce the loss by shrinking its own $\log p_S$, which is unrelated to whether the sampled rollout was correct. Gradients therefore flow only through the outer $\log p_\theta$ factor in Eq.~\ref{eq:rlsd_loss}. The whole token-level advantage $A_i\,w_{i,t}$ (and later $A_i\,w_{i,t}\,\omega_{i,t}$ for EGRSD) is treated as a constant multiplier during backpropagation.

\paragraph{Advantage whitening.}
Rewards $r_i$ are whitened using a running reward mean $\bar r_{<s}$ and standard deviation $\sigma_{r,<s}$ maintained across training steps via Welford's algorithm:
\begin{equation}
\label{eq:advantage}
    A_i = (r_i - \bar r_{<s})/\sigma_{r,<s}.
\end{equation}
During the first ten steps, before the running statistics become reliable, we use a constant warm-up baseline of $0.5$, i.e.\ $A_i = r_i - 0.5$. This matches the implementation used by all compared methods under our shared trainer and isolates method-level differences from advantage-estimation variance.

\paragraph{Relation to GRPO.}
GRPO~\citep{shao2024deepseekmath} also broadcasts a sequence-level advantage to every token, but estimates the advantage from group-relative outcomes within each mini-batch and applies it with a uniform per-token weight. RLSD can be viewed as GRPO with the uniform weight replaced by the teacher--student likelihood ratio $w_{i,t}$, and EGRSD as RLSD with an additional teacher-confidence factor $\omega_{i,t}$.

\section{Derivations for the main-text remarks}
\label{app:derivations}

This appendix supplies the short derivations underlying the two methodological remarks in \S\ref{sec:egrsd} and \S\ref{sec:clegrsd}.

\subsection{Geometric interpretation of the linear gate}
\label{app:linear_gate}

We view the token-level update as a one-dimensional shrinkage problem on a signal-plus-noise reference model. The model is interpretive and is used only to motivate the form of the gate, not to establish a tight bound on the true loss surface. Fix a token position $(i,t)$ and write the raw update magnitude as $y=\mu+\epsilon$, where $\mu$ is a latent "useful" component and $\epsilon$ is a zero-mean noise term with variance $\sigma^2$. A multiplicative gate $\hat\mu=\omega y$ minimizes the mean-squared error $\mathbb{E}[(\hat\mu-\mu)^2]$ at
\begin{equation}
    \omega^\star \;=\; \frac{\mu^2}{\mu^2+\sigma^2} \;=\; \frac{1}{1+\sigma^2/\mu^2}.
    \label{eq:mse_shrinkage}
\end{equation}
To link teacher entropy to the noise-to-signal ratio, we adopt the simplest monotone proxy $\sigma^2/\mu^2 \le a_0\,\widehat{H}_{i,t}$ for some constant $a_0>0$. Since $\omega^\star$ is strictly decreasing in $\sigma^2/\mu^2$, the proxy's saturating case $\sigma^2/\mu^2 = a_0\widehat{H}$ yields the \emph{worst-case shrinkage bound}
\begin{equation}
    \omega^\star_{a_0}(\widehat{H}) \;=\; \frac{1}{1+a_0\widehat{H}}, \qquad \widehat{H}\in[0,1],
    \label{eq:reference_curve}
\end{equation}
which lower-bounds the true MSE-optimal shrinkage: $\omega^\star(\widehat{H}) \ge \omega^\star_{a_0}(\widehat{H})$. We use $\omega^\star_{a_0}$ as the reference curve because it is the most aggressive shrinkage compatible with the proxy. Matching it at the endpoints is therefore a conservative design.

\paragraph{Endpoint chord.}
The linear gate $\omega(\widehat{H})=1-\gamma\widehat{H}$ is the secant of $\omega^\star_{a_0}$ at the two boundary points $\widehat{H}\in\{0,1\}$. Matching endpoints,
\begin{equation}
    \omega(0)=\omega^\star_{a_0}(0)=1
    \quad\text{holds automatically,}\qquad
    \omega(1)=\omega^\star_{a_0}(1)\;\;\Longleftrightarrow\;\;
    1-\gamma=\frac{1}{1+a_0},
\end{equation}
which solves to
\begin{equation}
    \gamma \;=\; \frac{a_0}{1+a_0} \;=\; \frac{\mathrm{NSR}_{\max}}{1+\mathrm{NSR}_{\max}},
    \label{eq:gamma_nsr}
\end{equation}
where $\mathrm{NSR}_{\max}:=a_0$ denotes the worst-case noise-to-signal ratio at $\widehat{H}=1$ under the proxy. Because $\omega^\star_{a_0}$ is strictly convex for $a_0>0$, the linear gate lies on or above the reference curve throughout $[0,1]$, with equality only at the two endpoints. We do not claim linearity is MSE-optimal. We claim it is the simplest affine form that matches the worst-case shrinkage bound at the extreme points of the normalized-entropy range, and that the resulting $\gamma$ admits the direct noise-to-signal reading in~Eq.~\eqref{eq:gamma_nsr}. The linear NSR proxy $\sigma^2/\mu^2\le a_0\widehat{H}$ is a heuristic: we make no quantitative commitment to the numerical value of $a_0$.

\subsection{Minimum as the extremal causal smoothing filter}
\label{app:min_pool}

\begin{definition}[Causal smoothing filter family]
\label{def:filter_family}
For a window $W\ge 1$, let $\mathcal{F}_W$ denote the class of functions $\phi:\mathbb{R}^{W+1}_{\ge 0}\to\mathbb{R}_{\ge 0}$ satisfying, for every input $(h_0,\ldots,h_W)$ and every $c\ge 0$: \textbf{(a)} \emph{per-argument monotonicity}, with $\phi$ non-decreasing in each coordinate separately; \textbf{(b)} \emph{conservativity}, with $\phi(h_0,\ldots,h_W)\le h_0$; \textbf{(c)} \emph{idempotency}, with $\phi(c,\ldots,c)=c$; and \textbf{(d)} \emph{causality}, where $\phi$ depends only on the current entry $h_0$ and the $W$ future entries $h_1,\ldots,h_W$.
\end{definition}

Condition (b) is what we require of the family: a lookahead replacement for the gate should never inflate the uncertainty attributed to a currently low-entropy (lock) position, since doing so would attenuate reliable supervision. Standard windowed averages violate (b) whenever a low-entropy token precedes a high-entropy span, and are therefore excluded.

\begin{lemma}[Pointwise lower bound]
\label{lem:min_lowerbound}
Every $\phi\in\mathcal{F}_W$ satisfies $\phi(h_0,\ldots,h_W)\ge \min_{0\le j\le W} h_j$.
\end{lemma}
\begin{proof}
Let $m:=\min_j h_j$. Since $h_j\ge m$ for every $j$, per-argument monotonicity (a) gives $\phi(h_0,\ldots,h_W)\ge\phi(m,\ldots,m)$. Idempotency (c) gives $\phi(m,\ldots,m)=m$. Combining the two yields $\phi(h_0,\ldots,h_W)\ge m$.
\end{proof}

\begin{proposition}[Extremal weight recovery at pivots]
\label{prop:min_extremal}
Instantiate CL-EGRSD with any $\phi\in\mathcal{F}_W$ by replacing $\widehat{H}_{i,t}$ in Eq.~\eqref{eq:entropy_weight_improved} by $\phi(H_{i,t},\ldots,H_{i,t+W})/H_{\max}^{\mathrm{batch}}$, and let the weight increment under the replacement be
\[
    \Delta^\phi_{i,t} \;:=\; \gamma\cdot\frac{h_0-\phi(h_0,\ldots,h_W)}{H_{\max}^{\mathrm{batch}}},
    \qquad h_j:=H_{i,t+j}.
\]
At any position where the gate is not saturated by the lower clip in Eq.~\eqref{eq:entropy_weight_improved},
\[
    \Delta^\phi_{i,t} \;\le\; \Delta^{\min}_{i,t}
    \qquad\text{for every } \phi\in\mathcal{F}_W.
\]
\end{proposition}
\begin{proof}
By Lemma~\ref{lem:min_lowerbound}, $\phi(h_0,\ldots,h_W)\ge\min_j h_j$, hence $h_0-\phi\le h_0-\min_j h_j$. Multiplying by $\gamma/H_{\max}^{\mathrm{batch}}\ge 0$ preserves the inequality.
\end{proof}

Pivot positions are by definition those with large $h_0-\min_j h_j$ (high current entropy, low lookahead entropy), so Proposition~\ref{prop:min_extremal} states that the minimum is the member of $\mathcal{F}_W$ that maximizes weight recovery there. Conservativity (b) simultaneously guarantees that at sustained high-entropy (fork) positions, where $\min_j h_j\approx h_0$, every member of the family (including the minimum) delivers essentially no recovery. The inequality $\Delta^\phi\le\Delta^{\min}$ is strict when $h_0>\min_j h_j$ (the non-degenerate pivot case). On constant windows $h_0=\cdots=h_W$, every $\phi\in\mathcal{F}_W$ collapses to the common value by idempotency and the weight recovery is zero across the family. This is the selectivity property CL-EGRSD exploits.

\section{Full experimental details}
\label{app:setup}

\paragraph{Baselines.}
We select comparison methods that are closely related to our approach and have publicly available training code compatible with our shared protocol: the no-train base model, supervised fine-tuning (SFT), Group Relative Policy Optimization (GRPO)~\citep{shao2024deepseekmath}, on-policy self-distillation (OPSD)~\citep{opsd}, and CRISP~\citep{opsdc}. While Reinforcement Learning with Self-Distillation (RLSD)~\citep{yang2026rlsd} is conceptually highly relevant, it lacks public training code at the time of writing. We therefore use our direction-aware baseline ($\gamma{=}0$) as its reference point in the ablations.

\paragraph{Training data.}
Our training data configuration matches that of OPSD~\citep{opsd}: the subset of OpenThoughts-114k~\citep{openthoughts114k} distilled from DeepSeek-R1~\citep{deepseekr1} reasoning traces and filtered to answer-verified samples~\citep{opsd_data}. Each sample provides a problem $x$ and a concise reference solution $s^\star$. Our data collator uses only the \texttt{problem} and \texttt{solution} columns, so the teacher context is $(x, s^\star, y_{<t})$ and the student context is $(x, y_{<t})$. All compared methods share the same training data and preprocessing, so accuracy differences isolate the effect of the loss function.

\paragraph{Evaluation benchmarks.}
\textbf{AIME 2024} / \textbf{AIME 2025} are the $30$-problem annual American Invitational Mathematics Examination contests. \textbf{HMMT 2025} is the $30$-problem February 2025 Harvard-MIT Mathematics Tournament. \textbf{MATH-500}~\citep{lightman2023let} is a held-out $500$-problem subset of the MATH competition dataset~\citep{hendrycks2021math}. \textbf{Minerva Math}~\citep{lewkowycz2022minerva} is a $272$-problem college-level STEM reasoning benchmark. \textbf{GSM8K}~\citep{cobbe2021gsm8k} is a $1{,}319$-problem grade-school arithmetic benchmark. Decoding and aggregation are described in \S\ref{sec:setup}.

\paragraph{Implementation details.}
All training and on-policy rollouts are conducted on a single $8\times$ NVIDIA A100-SXM4-80GB server ($640$\,GiB aggregate HBM, full-mesh NVLink). Total wall-clock time across all reported runs is approximately $24$ hours on this machine.
For the software environment, training and evaluation are implemented in PyTorch $2.10$ (CUDA $12.8$, NVIDIA driver $580$) using Transformers $4.57$ and PEFT~\citep{peft} $0.18$. We employ Hugging Face Accelerate $1.12$ with the \texttt{MULTI\_GPU} (DDP) backend. FlashAttention-$2$~\citep{dao2024flashattention2} (version $2.8.3$) is utilized to accelerate long-context forward passes. On-policy sampling and evaluation decoding are supported by vLLM $0.18$ in colocate mode. The detailed training hyperparameters are summarized in Table~\ref{tab:train_full}.

\begin{table}[ht]
\centering
\small
\caption{Training configuration shared by all methods on both model sizes. Only the loss and the loss-specific switches ($\gamma$, $W$, entropy mode) differ across methods.}
\begin{tabular}{lc}
\toprule
\textbf{Hyperparameter} & \textbf{Value} \\
\midrule
Optimizer & AdamW ($\beta_1{=}0.9, \beta_2{=}0.999$) \\
Learning rate & $5\times 10^{-6}$ \\
Max grad norm & $0.1$ \\
Per-device batch size & $4$ \\
Grad-accum steps & $1$ \\
Effective batch size & $32$ \\
Total training steps & $100$ \\
Checkpoint interval & every $25$ steps ($25/50/75/100$) \\
Precision & BF16 mixed precision \\
Sequence-window cap \texttt{max\_length} & $20{,}000$ tokens \\
\midrule
On-policy sampling (vLLM) & $T{=}1.1$, top-$p$ $0.95$, top-$k$ $20$ \\
Max per-step completion length & $1{,}024$ tokens \\
Rollouts per prompt \texttt{num\_rollouts} & $1$ \\
Magnitude clip $\varepsilon$ & $0.2$ \\
Length-shaping coefficient $\beta_L$ & $0.5$ (applied only on correct completions) \\
Length-shaping form & linear, $1-|y_i|/L_\mathrm{max}$ \\
Wrong-answer penalty & $0$ (incorrect $\Rightarrow r_i{=}0$) \\
Fixed (frozen) teacher & true \\
% Reason-first teacher prompt & true (two-turn chat template with $s^\star$) \\
\bottomrule
\end{tabular}
\label{tab:train_full}
\end{table}

\paragraph{Reward details.}
We use a verifier-based trajectory-level reward with optional length shaping:
\begin{equation*}
    r_i = \mathbb{1}[y_i\,\text{correct}]\cdot\big(1+\beta_L\cdot(1-|y_i|/L_\mathrm{max})\big),
\end{equation*}
where $\beta_L=0.5$ is the length-shaping coefficient and $L_\mathrm{max}$ is the maximum completion length. Incorrect and unverifiable completions receive $r_i=0$, and we do not use negative rewards. Advantage whitening and the teacher/student implementation are described in Appendix~\ref{app:background_full}.

\paragraph{Choice of $\gamma$ in the main tables.}
We adopt $\gamma=0.3$ for EGRSD at both model sizes and for CL-EGRSD on Qwen3-4B, so that the confidence-gate coefficient is matched between EGRSD and CL-EGRSD on the smaller model. On Qwen3-8B, CL-EGRSD instead uses $(\gamma, W) = (1.0, 5)$, drawn from the joint sweep in Table~\ref{tab:cl_grid_full} and discussed further in Appendix~\ref{app:recipe}. The full hyperparameter recipe is summarized in Table~\ref{tab:recipe}.

\section{Hyperparameters}
\label{app:recipe}

Table~\ref{tab:recipe} lists the per-method hyperparameters used for EGRSD and CL-EGRSD in the main tables (Tables~\ref{tab:main_4b_improved} and~\ref{tab:main_8b_improved}). On Qwen3-4B we use $\gamma=0.3$ for both methods and $W=3$ for CL-EGRSD, drawn from the $\gamma$ sweep in Table~\ref{tab:gamma_improved} and the matched-$\gamma$ lookahead sweep in Table~\ref{tab:window_improved}. On Qwen3-8B we use $\gamma=0.3$ for EGRSD and $(\gamma, W) = (1.0, 5)$ for CL-EGRSD. The latter is drawn from the joint sweep in Table~\ref{tab:cl_grid_full}, attains the highest Avg. in our six-benchmark evaluation, and isolates the contribution of the causal-lookahead window from entropy-coefficient strength through comparison with the matched $W{=}0$ reference. Stronger-suppression configurations for Qwen3-4B were not evaluated. Extending the joint $\gamma$--$W$ sweep to 4B is left as future work.

\begin{table}[h]
\centering
\small
\caption{Per-method hyperparameters for the main-table EGRSD and CL-EGRSD results on Qwen3-4B and Qwen3-8B.}
\label{tab:recipe}
\begin{tabular}{l l c c}
\toprule
\textbf{Model} & \textbf{Method} & \textbf{$\gamma$} & \textbf{$W$} \\
\midrule
Qwen3-4B & EGRSD      & $0.3$ & --- \\
Qwen3-4B & CL-EGRSD   & $0.3$ & $3$ \\
Qwen3-8B & EGRSD      & $0.3$ & --- \\
Qwen3-8B & CL-EGRSD   & $1.0$ & $5$ \\
\bottomrule
\end{tabular}
\end{table}

\section{Training dynamics}
\label{app:dynamics}

Figure~\ref{fig:gradnorm_8b} reports the per-step pre-clip gradient norm on Qwen3-8B for four representative dense-logging reruns over the $100$-step budget: the baselines OPSD and CRISP, and our EGRSD and CL-EGRSD. We omit the corresponding loss curves because the shared trainer logs a signed advantage-weighted surrogate whose magnitude is not directly comparable across methods, so we do not make cross-method loss comparisons.

\begin{figure}[h]
\centering
\includegraphics[width=0.85\linewidth]{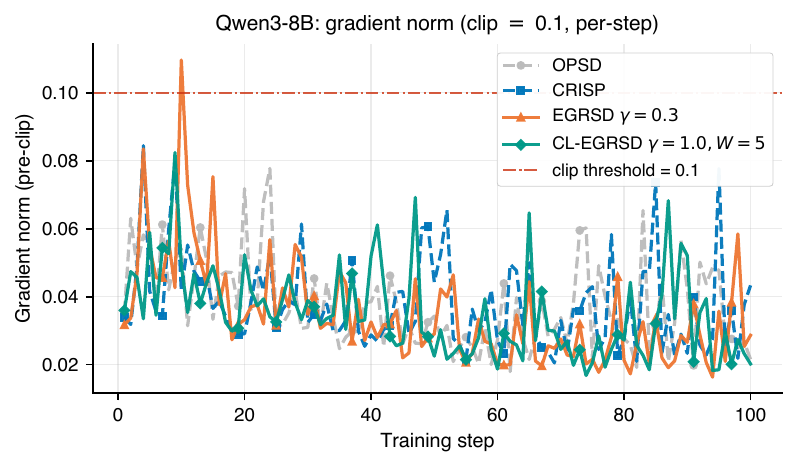}
\caption{Per-step pre-clip gradient norm on Qwen3-8B. The red dash-dot line is the clipping threshold ($0.1$). Clipping is rarely active: only one EGRSD step slightly exceeds the threshold, while the remaining runs stay below it.}
\label{fig:gradnorm_8b}
\end{figure}

\textbf{Gradient norm.} The dense traces show that all four methods operate in a narrow pre-clip range for almost the entire run. OPSD and CRISP stay mostly in the $0.02$--$0.08$ band, while EGRSD and CL-EGRSD have similar average magnitudes (mean pre-clip norm $\approx 0.035$). The only visible clipping event is a single EGRSD spike early in training, which reaches $\approx 0.11$. CL-EGRSD remains below the $0.1$ threshold throughout. Thus, the entropy-weighted objectives do not systematically inflate update magnitudes, and gradient clipping acts as an occasional safeguard rather than a persistent constraint.

% \section{Token Efficiency (Qwen3-8B)}
% \label{app:efficiency_8b}

% Figure~\ref{fig:token_efficiency_8b} visualizes the accuracy--length trade-off on Qwen3-8B. Unlike the 4B setting (Table~\ref{tab:efficiency_full}), where EGRSD and CL-EGRSD simultaneously improve accuracy and compress generations relative to the no-train baseline, on 8B our methods trade slightly more tokens for meaningfully higher accuracy. The Pareto frontier of the accuracy--length trade-off on 8B is spanned by the no-train baseline (shorter but lower accuracy), EGRSD (ours), and CL-EGRSD (ours) at the high-accuracy end. All trainable baselines (SFT, GRPO, OPSD, CRISP) are Pareto-dominated.

% \begin{figure}[h]
% \centering
% \includegraphics[width=0.85\linewidth]{figures/fig_token_efficiency_8b.pdf}
% \caption{Token efficiency on Qwen3-8B. Points show each method's average generation length versus Avg. accuracy across six benchmarks. The green dashed line traces the Pareto frontier (minimize length, maximize accuracy), which is spanned by the no-train baseline, EGRSD (ours), and CL-EGRSD (ours); all trainable baselines (SFT, GRPO, OPSD, CRISP) are dominated.}
% \label{fig:token_efficiency_8b}
% \end{figure}

\section{Extended CL-EGRSD ablation}
\label{app:clgrid}

To isolate the interaction between lookahead window $W$ and entropy coefficient $\gamma$, we ran an additional CL-EGRSD sweep on Qwen3-8B. We evaluate selected configurations on AIME24, AIME25, HMMT25, and Minerva with $K=4$. MATH500 and GSM8K are excluded because they saturate across the method family. Numerical results are reported in Table~\ref{tab:cl_grid_full}, and Figure~\ref{fig:gamma_w_multibench_full} visualizes the full $\gamma\times W$ grid as per-benchmark heatmaps.

\begin{table}[h]
\centering
\small
\caption{CL-EGRSD extended ablation on Qwen3-8B, suppress mode.}
\label{tab:cl_grid_full}
\begin{tabular}{l c c c c}
\toprule
\textbf{Config} ($\gamma,W$) & \textbf{AIME24} & \textbf{AIME25} & \textbf{HMMT25} & \textbf{Minerva} \\
\midrule
\multicolumn{5}{l}{\textbf{8B, suppress mode ($\gamma=1.0$ window sweep)}} \\
EGRSD $(1.0,0)$ ref  & 76.39 & 66.67 & 47.50 & 34.56 \\
CL-EGRSD $(1.0,3)$   & 74.17 & 65.83 & 37.50 & 34.56 \\
CL-EGRSD $(1.0,7)$   & 74.44 & 70.83 & 43.33 & \textbf{35.39} \\
\rowcolor{bestcolor}
CL-EGRSD $(1.0,5)$   & \textbf{77.22} & \textbf{70.00} & \textbf{52.50} & 34.83 \\
\midrule
\multicolumn{5}{l}{\textbf{8B, suppress mode ($W=5$ gamma sweep)}} \\
EGRSD $(0.5,0)$ ref  & 75.00 & 67.50 & 43.33 & 34.56 \\
CL-EGRSD $(0.3,5)$   & 74.72 & 69.17 & 45.00 & \textbf{35.11} \\
CL-EGRSD $(0.5,5)$   & 76.11 & 65.00 & 47.50 & 34.38 \\
\rowcolor{bestcolor}
CL-EGRSD $(1.0,5)$   & \textbf{77.22} & \textbf{70.00} & \textbf{52.50} & 34.83 \\
\bottomrule
\end{tabular}
\end{table}

\begin{figure}[t]
\centering
\includegraphics[width=\linewidth]{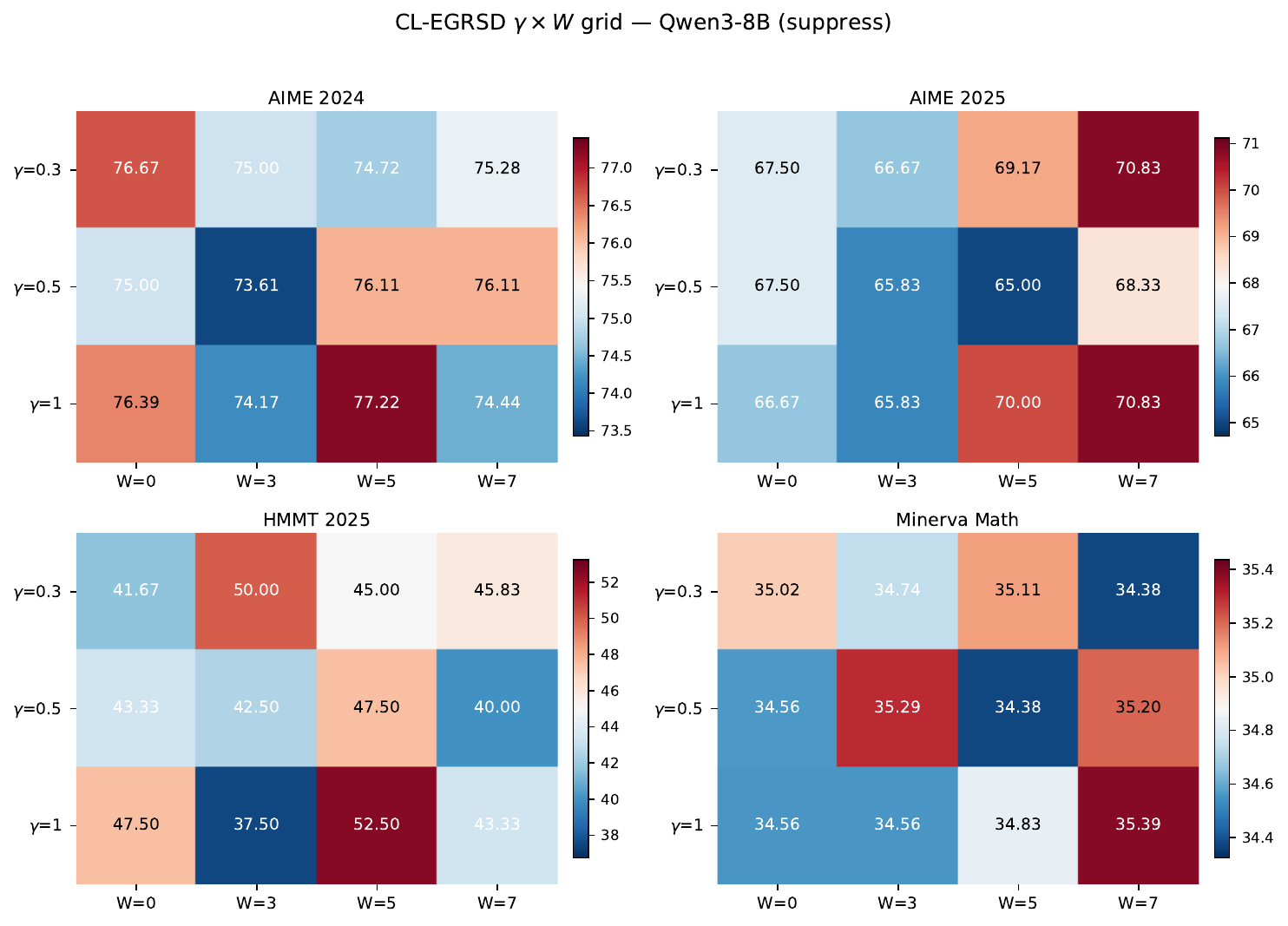}
\caption{Extended CL-EGRSD $\gamma\times W$ grid on Qwen3-8B, suppress mode. All numbers are avg@4 (\%). The $W=0$ column is the matched EGRSD reference. Other columns are CL-EGRSD. Per-panel color scales emphasize within-benchmark differences.}
\label{fig:gamma_w_multibench_full}
\end{figure}

At strong suppression, $(\gamma,W)=(1.0,5)$ improves over the matched $W=0$ reference on all four evaluated benchmarks, with the largest gains on HMMT25 ($47.50\to52.50$) and AIME25 ($66.67\to70.00$). This focused sweep motivates jointly tuning $\gamma$ and $W$, but it uses a restricted benchmark subset, so the main headline remains the six-benchmark result.

\section{Qualitative pivot examples}
\label{app:qualitative_pivots}

This appendix provides three qualitative views of the pivot/fork distinction on held-out reasoning windows. Figure~\ref{fig:token_entropy_global} gives a global view of the lock/fork/pivot token regimes demonstrated in the main-text Figure~\ref{fig:token_entropy_motivation}. Figure~\ref{fig:token_entropy_case_full} visualizes pivot rescue directly, showing per-token current entropy, five-token lookahead entropy, and the corresponding EGRSD and CL-EGRSD weights on representative reasoning windows. Figure~\ref{fig:token_entropy_annotated} further overlays the top-$K$ local entropy peaks on two Minerva completions together with their 4-token left context, verifying that most annotated peaks are transient (i.e., pivots whose lookahead entropy drops rapidly) rather than sustained forks.

\begin{figure}[!h]
\centering
\includegraphics[width=\linewidth]{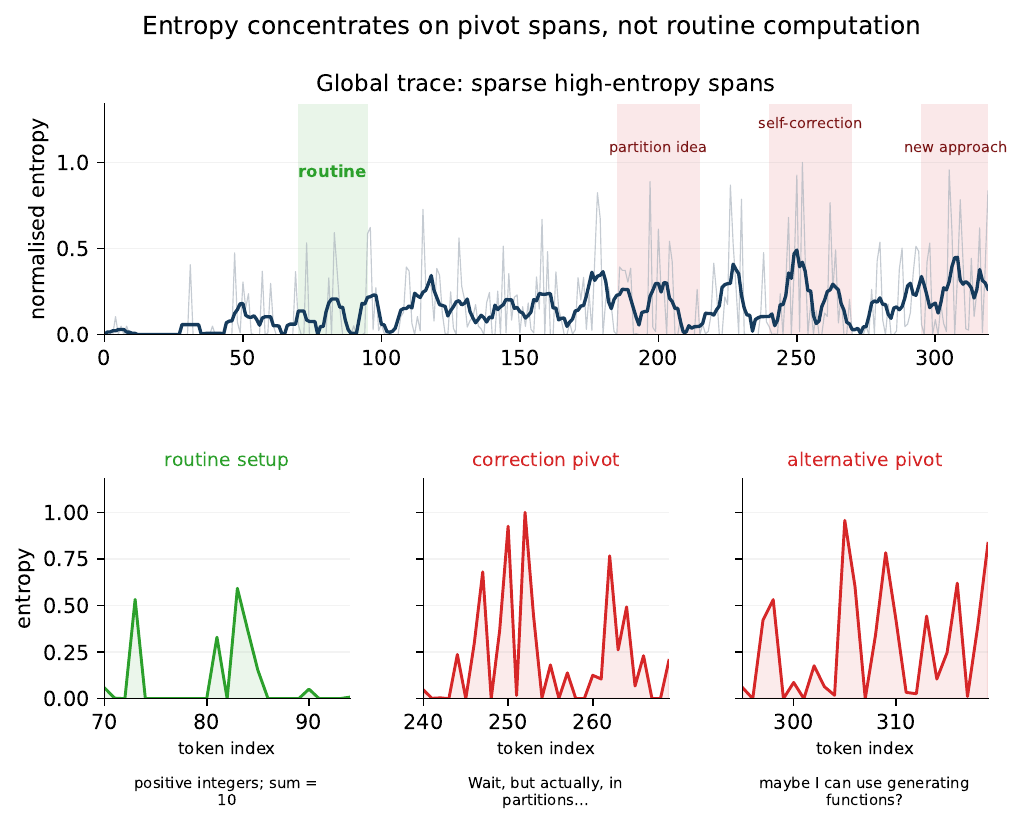}
\caption{Global entropy trace on a held-out reasoning problem (Qwen3-4B), illustrating at full-trace scale the same lock/fork/pivot token regimes demonstrated in the main-text Figure~\ref{fig:token_entropy_motivation}. \textbf{Top}: per-token normalized entropy across the $\sim$320-token reasoning trace, with a smoothed envelope. The green span marks a routine-computation (lock) region and the three red spans mark strategy-pivot transitions. \textbf{Bottom}: zoomed views of one routine span and two pivot spans, each with a short text snippet from the generated trace.}
\label{fig:token_entropy_global}
\end{figure}

\begin{figure}[h]
\centering
\includegraphics[width=\linewidth]{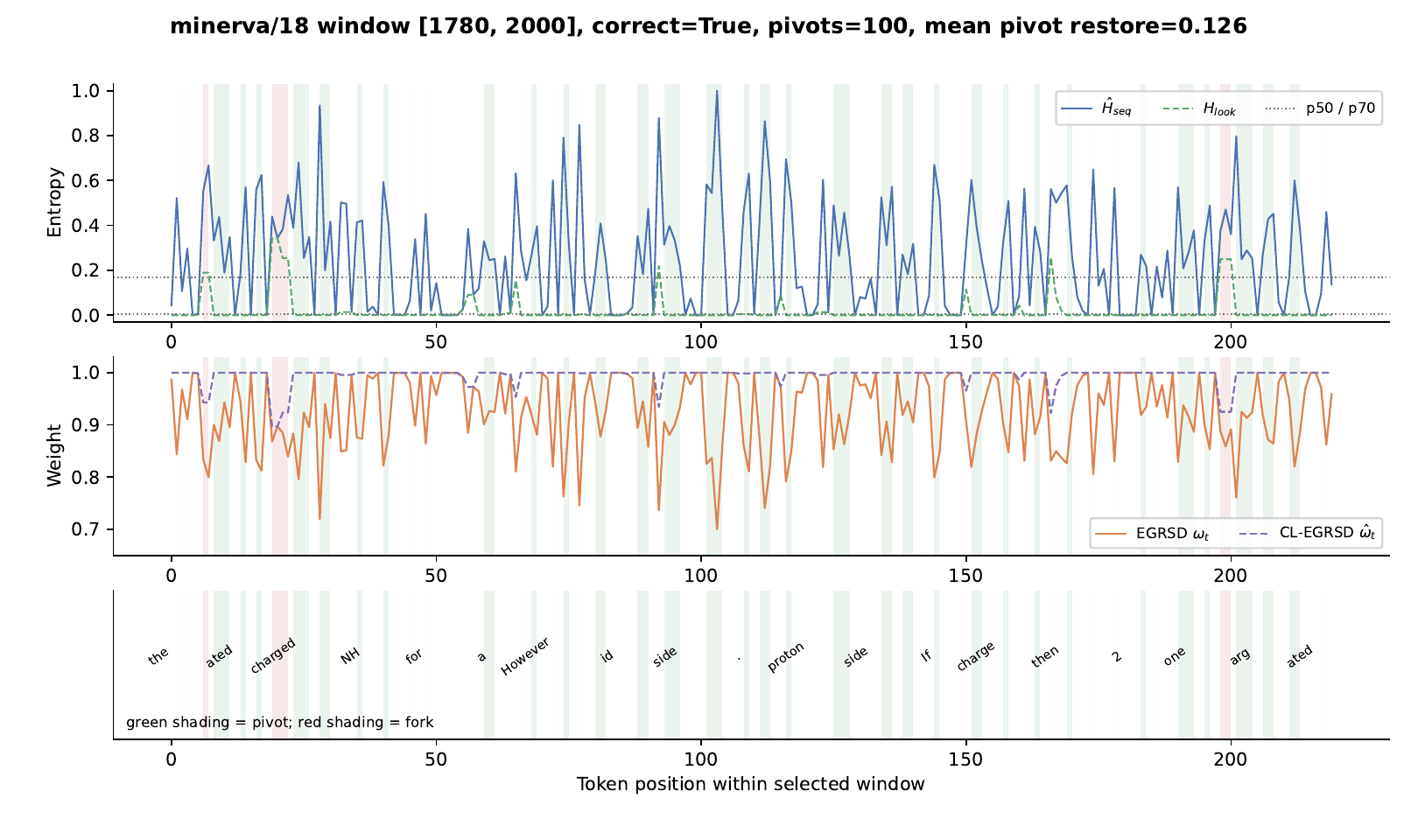}
\caption{Qualitative token-level examples of pivot rescue. Each page shows current entropy, five-token lookahead entropy, EGRSD weight, and CL-EGRSD weight for a selected reasoning window. Green spans mark pivot tokens where high current entropy resolves into low future entropy and CL-EGRSD restores the token weight. Red spans mark sustained forks. Token snippets are sparse anchors rather than full transcripts.}
\label{fig:token_entropy_case_full}
\end{figure}

\begin{figure}[!h]
\centering
\includegraphics[width=\linewidth]{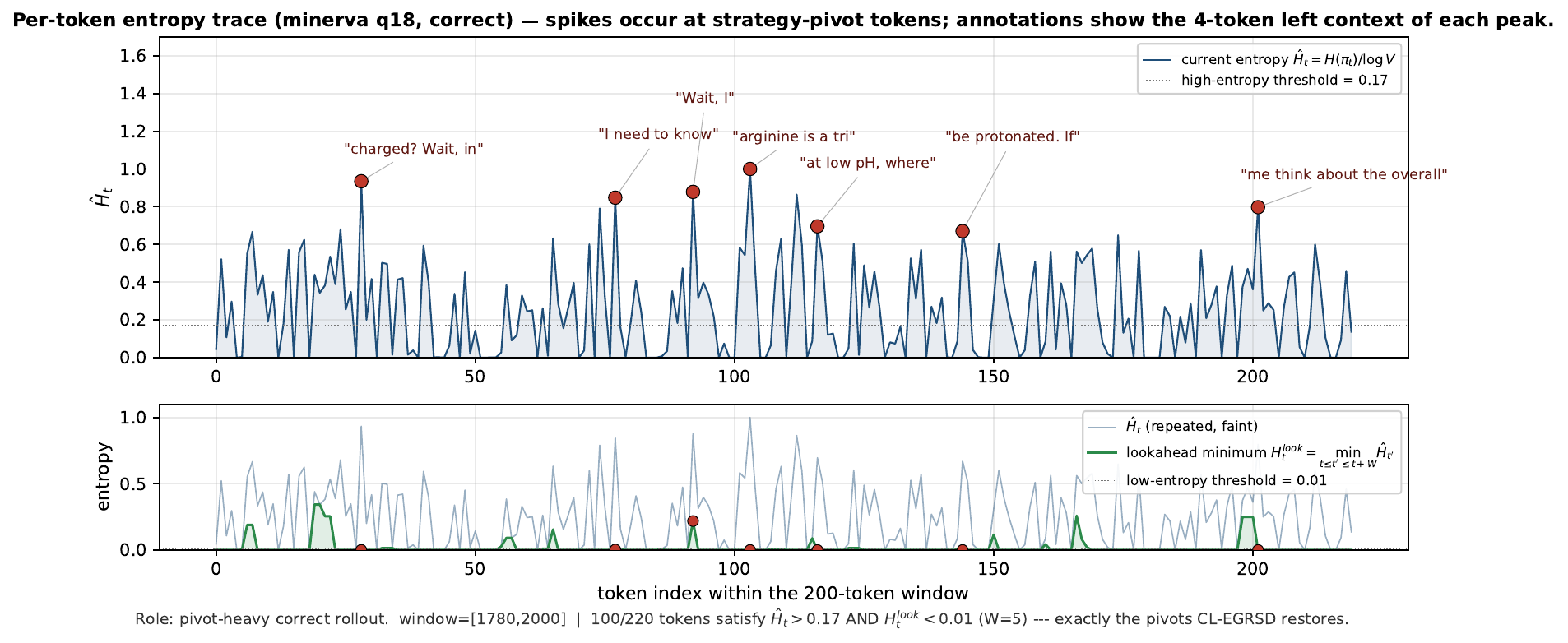}
\caption{Per-token entropy trace with annotated transition points on two Minerva reasoning windows (Qwen3-4B). \textbf{Top row}: current entropy $\widehat{H}_{i,t}$ with the top-$K$ local peaks highlighted and annotated with the 4-token left context (red dots). The annotated peaks align with discourse transition markers and local shifts in the generated derivation. \textbf{Bottom row}: the five-token causal lookahead $\widehat{H}_{i,t}^{\mathrm{CL}}$ with the same peak indices overlaid. At $W=5$, $\widehat{H}_{i,t}^{\mathrm{CL}}\leq 0.08$ for $13$ of the $14$ annotated peaks, matching the transient-transition signature targeted by CL-EGRSD.}
\label{fig:token_entropy_annotated}
\end{figure}

\section{AceReason-Nemotron-7B cross-architecture diagnostic}
\label{app:acereason_diag}

We additionally evaluate the methods on \textbf{AceReason-Nemotron-7B} ~\citep{chen2025acereason}, a strong reasoning-tuned external base, with results detailed in Table~\ref{tab:acereason_diag_full}. This diagnostic is intentionally separate from the Qwen3 main comparison: AceReason has a different post-training recipe and a longer native reasoning style, and the no-training model is already competitive. 

\begin{table}[!h]
\centering
\small
\caption{AceReason-Nemotron-7B cross-architecture diagnostic. All accuracy metrics are percentages. Best and runner-up values are shown in \textbf{bold} and \underline{underline}.}
\label{tab:acereason_diag_full}
\resizebox{\textwidth}{!}{
\begin{tabular}{l c c c c c c c c}
\toprule
\textbf{Method} & \textbf{AIME24} & \textbf{AIME25} & \textbf{HMMT25} & \textbf{MATH500} & \textbf{Minerva} & \textbf{GSM8K} & \textbf{Avg.} & \textbf{AvgLen} \\
\midrule
No train & 65.56 & 49.17 & \textbf{37.50} & \textbf{87.00} & \underline{33.09} & 91.49 & 60.63 & 8{,}382 \\
SFT & \textbf{68.06} & 49.17 & 30.00 & 86.35 & 32.35 & 91.64 & 59.60 & 8{,}360 \\
GRPO~\citep{shao2024deepseekmath} & 67.50 & \underline{50.83} & 34.17 & 86.85 & 32.31 & 91.49 & 60.52 & 8{,}218 \\
OPSD~\citep{opsd} & \underline{67.78} & 49.17 & 35.83 & 86.15 & 32.35 & 89.48 & 60.13 & 8{,}259 \\
CRISP~\citep{opsdc} & 67.78 & 49.17 & 35.83 & 86.70 & \textbf{33.09} & \underline{91.94} & \underline{60.75} & \textbf{7{,}908} \\
\rowcolor{ourcolor}
\textbf{CL-EGRSD} ($\gamma{=}0.5$) & 67.22 & \textbf{52.50} & \underline{35.83} & \underline{86.90} & 32.35 & \textbf{93.91} & \textbf{61.45} & \underline{7{,}943} \\
\bottomrule
\end{tabular}
}
\end{table}

The no-training AceReason model is already strong (60.63), and most trained baselines fail to improve over it: SFT degrades to 59.60, while GRPO (60.52), OPSD (60.13), and CRISP (60.75) cluster around the no-training level. CL-EGRSD at $\gamma=0.5$ is the only trained method in this diagnostic that clearly exceeds the no-training baseline, reaching 61.45 avg@4 (+0.82 over no-train) with shorter completions (7{,}943 vs. 8{,}382 tokens) and the highest GSM8K score (93.91, +2.42). This suggests that on reasoning-tuned external bases with a different post-training recipe, uniform-weight distillation objectives struggle to add value, while the entropy-gated update preserves a useful learning signal. The entropy coefficient likely needs to be retuned per base, since the same $\gamma$ selected for Qwen3 need not transfer unchanged.

\section{Weak-base cross-architecture diagnostic on Olmo-3-7B Base}
\label{app:olmo_base_diag}

We also run a weak-base cross-architecture diagnostic on \textbf{Olmo-3-7B Base} ~\citep{groeneveld2024olmo}, a non-reasoning-tuned external base model. Absolute performance is much lower than on Qwen3 or on reasoning-tuned external bases, so this experiment is not intended as a headline accuracy comparison. Instead, it tests negative transfer: whether an on-policy distillation objective can avoid making a weak external base worse. We report all six benchmark scores, the macro average, and the average generation length in Table~\ref{tab:olmo_base_diag_full}.

\begin{table}[!h]
\centering
\small
\caption{Weak-base diagnostic on Olmo-3-7B Base. All accuracy metrics are percentages. Best and runner-up values are shown in \textbf{bold} and \underline{underline}.}
\label{tab:olmo_base_diag_full}
\resizebox{\textwidth}{!}{
\begin{tabular}{l c c c c c c c c}
\toprule
\textbf{Method} & \textbf{AIME24} & \textbf{AIME25} & \textbf{HMMT25} & \textbf{MATH500} & \textbf{Minerva} & \textbf{GSM8K} & \textbf{Avg.} & \textbf{AvgLen} \\
\midrule
No train & 8.33 & \textbf{13.33} & \textbf{7.50} & 50.65 & 10.48 & 62.64 & 25.49 & \underline{7{,}203} \\
SFT & \underline{14.17} & 10.83 & 5.00 & \underline{51.80} & \textbf{12.04} & \textbf{65.26} & \underline{26.52} & 7{,}845 \\
GRPO~\citep{shao2024deepseekmath} & \textbf{16.67} & 9.17 & \underline{5.83} & 50.75 & 10.48 & 62.81 & 25.95 & 7{,}585 \\
OPSD~\citep{opsd} & 13.06 & \textbf{13.33} & 5.00 & \textbf{53.90} & 10.57 & 55.67 & 25.25 & 7{,}311 \\
CRISP~\citep{opsdc} & 8.33 & 10.83 & \textbf{7.50} & 50.00 & 10.11 & 55.27 & 23.67 & \textbf{6{,}432} \\
\rowcolor{ourcolor}
\textbf{CL-EGRSD} & \textbf{16.67} & \underline{12.50} & \underline{5.83} & {51.15} & \underline{10.85} & \underline{63.34} & \textbf{26.72} & {7{,}439} \\
\bottomrule
\end{tabular}
}
\end{table}

\begin{table}[h]
\centering
\small
\caption{Aggregate deltas for the Olmo-3-7B Base diagnostic. $\Delta_{\rm base}$ is relative to the no-training model. $\Delta_{\rm SFT}$ is relative to SFT on the same base.}
\label{tab:olmo_base_delta}
\begin{tabular}{l c c c}
\toprule
\textbf{Method} & \textbf{Avg.} & \textbf{$\Delta_{\rm base}$} & \textbf{$\Delta_{\rm SFT}$} \\
\midrule
No train & 25.49 & -- & $-1.03$ \\
SFT & 26.52 & $+1.03$ & -- \\
GRPO~\citep{shao2024deepseekmath} & 25.95 & $+0.46$ & $-0.57$ \\
OPSD~\citep{opsd} & 25.25 & $-0.24$ & $-1.27$ \\
CRISP~\citep{opsdc} & 23.67 & $-1.82$ & $-2.85$ \\
\rowcolor{ourcolor}
\textbf{CL-EGRSD} & \textbf{26.72} & \textbf{$+1.23$} & \textbf{$+0.20$} \\
\bottomrule
\end{tabular}
\end{table}

As summarized in Table~\ref{tab:olmo_base_delta}, CL-EGRSD is the only trainable method in this weak-base diagnostic that improves over both the no-training model and SFT. In contrast, OPSD and CRISP exhibit negative transfer relative to the no-training base, and GRPO remains below SFT. This supports the interpretation that the entropy gate stabilizes training under weak external bases by attenuating harmful token-level updates when the privileged teacher signal is less reliable for the student's native distribution.

\phantomsection
\let\addcontentsline\originaladdcontentsline
\addcontentsline{toc}{section}{Limitations}
\let\addcontentsline\relax
\def\addcontentsline#1#2#3{}
\section*{Limitations}
\label{sec:limitations}

Our evaluation is restricted to mathematical reasoning, and transfer to code, agentic reasoning, or open-ended reasoning remains untested. The mechanism diagnostics use held-out benchmarks and therefore condition the privileged teacher on a proxy reference solution rather than the original training-set reference. A further limitation is that teacher Shannon entropy is only an operational confidence measure and an imperfect proxy for token-level supervision reliability: high entropy can arise from epistemic ambiguity in the reasoning path, but also from aleatoric variation such as synonymous phrasing or vocabulary-level underspecification. Relatedly, EGRSD and CL-EGRSD introduce additional hyperparameters ($\gamma$ and, for CL-EGRSD, the lookahead window $W$) beyond those of existing OPSD-style objectives. The per-model settings we use are listed in Appendix~\ref{app:recipe}, and a scale-aware or schedule-based $\gamma$ controller is a natural direction for future work.

\section{Algorithm}
\label{app:algorithm}

Algorithm~\ref{alg:egrsd} summarizes a single training step of EGRSD / CL-EGRSD. The teacher $p_T$ is the frozen base model conditioned on the privileged context $(x,s^\star)$, and the student $p_S=p_\theta$ is conditioned on $(x)$ only. The teacher's weights are not updated through LoRA. Entropy normalization is shared across tokens in the current minibatch, and the lookahead window $W$ distinguishes EGRSD ($W{=}0$) from CL-EGRSD ($W{>}0$). Code, LoRA adapters, and evaluation artifacts will be released publicly.

\begin{algorithm}[h]
\small
\caption{EGRSD / CL-EGRSD training step (shared direction--magnitude--confidence form).}
\label{alg:egrsd}
\begin{algorithmic}[1]
\Require Batch of problems $\{x_i\}$ with privileged solutions $\{s^\star_i\}$; student $p_\theta$; frozen teacher $p_T$ (shared backbone, adapter disabled); entropy coefficient $\gamma$; lookahead window $W$ (set $W{=}0$ for EGRSD); clip $\varepsilon$; length-shape $\beta_L$; running reward stats $(\bar r,\sigma_r)$.
\For{each problem $x_i$ in batch}
    \State Sample on-policy completion $y_i \sim p_\theta(\cdot \mid x_i)$ via vLLM. \Comment{Rollout}
    \State Compute reward $r_i \gets \mathbb{1}[y_i\text{ correct}]\cdot\big(1+\beta_L(1-|y_i|/L_{\max})\big)$.
    \State Whiten advantage $A_i \gets (r_i - \bar r)/\sigma_r$; set direction $D_i \gets \mathrm{sign}(A_i)$.
\EndFor
\State Forward pass: collect $\log p_T(y_{i,t}\mid x_i,s^\star_i,y_{i,<t})$ and $\log p_S(y_{i,t}\mid x_i,y_{i,<t})$; use \texttt{stop-gradient} on $p_S$.
\ForAll{completion tokens $(i,t)$}
    \State Compute teacher entropy $H_{i,t} \gets -\sum_v p_T(v\mid \cdot)\log p_T(v\mid \cdot)$.
\EndFor
\If{CL-EGRSD ($W>0$)}
    \State $H_{i,t} \gets \min_{j\in[t,\,\min(t+W,T_i)]} H_{i,j}$ \Comment{Causal lookahead replacement.}
\EndIf
\State Batch-normalize: $\widehat{H}_{i,t} \gets H_{i,t} / \max\!\big(\max_{(j,k)} H_{j,k},\,1\big)$.
\State Confidence gate: $\omega_{i,t} \gets \mathrm{clip}(1-\gamma\,\widehat{H}_{i,t},\,0.1,\,1)$.
\State Directional log-ratio: $\delta_{i,t} \gets \log p_T(y_{i,t}\mid\cdot) - \log p_S(y_{i,t}\mid\cdot)$.
\State Magnitude: $w_{i,t} \gets \mathrm{clip}\!\big(\exp(D_i\,\delta_{i,t}),\,1-\varepsilon,\,1+\varepsilon\big)$.
\State Token-level advantage: $\widehat{A}_{i,t} \gets A_i\cdot w_{i,t}\cdot \omega_{i,t}$ (all stop-gradient).
\State Loss: $\mathcal{L} \gets -\tfrac{1}{\sum m_{i,t}}\sum m_{i,t}\,\widehat{A}_{i,t}\,\log p_\theta(y_{i,t}\mid x_i,y_{i,<t})$.
\State Gradient step: $\theta \gets \theta - \eta\,\mathrm{clip\text{-}norm}(\nabla_\theta\mathcal{L},\,0.1)$.
\State Update running reward statistics $(\bar r,\sigma_r)$ via Welford's algorithm.
\end{algorithmic}
\end{algorithm}

\end{document}